\documentclass{article}

\usepackage[preprint]{neurips_2020}

\usepackage[utf8]{inputenc} 
\usepackage[T1]{fontenc}    
\usepackage{hyperref}       
\usepackage{url}            
\usepackage{booktabs}       
\usepackage{amsfonts}       
\usepackage{nicefrac}       
\usepackage{microtype}      
\usepackage[pdftex]{graphicx}
\usepackage{xspace}
\usepackage{multirow}
\usepackage{enumitem}\usepackage{algorithm}
\usepackage{algorithmic}
\usepackage{todonotes}

\usepackage{natbib}
\usepackage{graphicx}
\usepackage{amsmath}
\usepackage{amssymb}
\usepackage{amsthm}
\usepackage{color}

\usepackage{natbib}
\usepackage{amsmath}
\usepackage{amssymb}
\usepackage{amsthm}

\usepackage{subcaption}
\usepackage{caption}

\usepackage{amsmath,amsfonts,amsthm,bm}

\usepackage{color}
\usepackage[normalem]{ulem}

\newcommand{\Bo}[1]{}
\newcommand{\comment}[1]{}

\newtheorem{theorem}{Theorem}
\newtheorem{lemma}[theorem]{Lemma}

\newtheorem{corollary}[theorem]{Corollary}

\def\eqref#1{equation~\ref{#1}}

\def\1{\bm{1}}

\DeclareMathAlphabet{\mathsfit}{\encodingdefault}{\sfdefault}{m}{sl}
\SetMathAlphabet{\mathsfit}{bold}{\encodingdefault}{\sfdefault}{bx}{n}

\DeclareMathOperator*{\argmin}{arg\,min}

\newcommand{\reals}{\mathbb{R}}

\newcommand{\state}{s} 
\newcommand{\obs}{x}
\newcommand{\latent}{z} 
\newcommand{\enc}{E}
\newcommand{\dyn}{F}
\newcommand{\dec}{D}

\usepackage{color}

\title{Control-Aware Representations for Model-based Reinforcement Learning}
\author{%
  \begin{tabular}{c@{\hspace*{.8cm}}c@{\hspace*{.8cm}}c}
  {\bf Brandon Cui} & {\bf Yinlam Chow} & {\bf Mohammad Ghavamzadeh}\\[.1cm]
  {\normalfont Facebook AI Research} & {\normalfont Google Research} &  {\normalfont Google Research} \\[.1cm]
  {\normalfont bcui@fb.com} & {\normalfont yinlamchow@google.com} & {\normalfont ghavamza@google.com} \\
  \end{tabular}
}
\date{April 2020}

\begin{document}

\maketitle

\begin{abstract}
\vspace{-0.08in}
A major challenge in modern reinforcement learning (RL) is efficient control of dynamical systems from high-dimensional sensory observations. 
Learning controllable embedding (LCE) is a promising approach that addresses this challenge by embedding the observations into a lower-dimensional latent space, estimating the latent dynamics, and utilizing it to perform control in the latent space. Two important questions in this area are how to learn a representation that is amenable to the control problem at hand, and how to achieve an end-to-end framework for representation learning and control. In this paper, we take a few steps towards addressing these questions. We first formulate a LCE model to learn representations that are suitable to be used by a policy iteration style algorithm  
in the latent space. We call this model {\em control-aware representation learning} (CARL). We derive a loss function for CARL that has close connection to the prediction, consistency, and curvature (PCC) principle for representation learning. We derive three implementations of CARL. In the {\em offline} implementation, we replace the locally-linear control algorithm (e.g.,~iLQR) used by the existing LCE methods with a RL algorithm, namely model-based soft actor-critic, and show that it results in significant improvement. In {\em online} CARL, we interleave representation learning and control, and demonstrate further gain in performance. Finally, we propose {\em value-guided} CARL, a variation in which we optimize a weighted version of the CARL loss function, where the weights depend on the TD-error of the current policy. We evaluate the proposed algorithms by extensive experiments on benchmark tasks and compare them with several LCE baselines.
\vspace{-0.085in}
\end{abstract}


\vspace{-0.1in}
\section{Introduction}
\label{sec:intro}
\vspace{-0.1in}

Control of non-linear dynamical systems is a key problem in control theory. Many methods have been developed with different levels of success in different classes of such problems. The majority of these methods assume that a model of the system is known and the underlying state of the system is low-dimensional and observable. These requirements limit the usage of these techniques in controlling dynamical systems from high-dimensional raw sensory data (e.g.,~image and audio), where the system dynamics is unknown, a scenario often seen in modern reinforcement learning (RL). 

Recent years have witnessed the rapid development of a large arsenal of model-free RL algorithms, such as DQN~\citep{mnih2013playing}, TRPO~\citep{schulman15}, PPO~\citep{schulman2017proximal}, and SAC~\citep{pmlr-v80-haarnoja18b}, with impressive success in solving high-dimensional control problems. However, most of this success has been limited to simulated environments (e.g.,~computer games), mainly due to the fact that these algorithms often require a large number of samples from the environment. This restricts their applicability in real-world physical systems, for which data collection is often a difficult process. On the other hand, model-based RL algorithms, such as PILCO~\citep{PILCO2011Deisenroth}, MBPO~\citep{MBPO2019Janner}, and Visual Foresight~\citep{ebert2018visual}, despite their success, still have many issues in handling the difficulties of learning a model (dynamics) in a high-dimensional (pixel) space. 

To address this issue, a class of algorithms have been developed that are based on learning a low-dimensional latent (embedding) space and a latent model (dynamics), and then using this model to control the system in the latent space. This class has been referred to as {\em learning controllable embedding} (LCE) and includes recently developed algorithms, such as E2C~\citep{E2C}, RCE~\citep{Banijamali2018RobustLC}, SOLAR~\citep{zhang19m}, PCC~\citep{Levine2020Prediction}, Dreamer~\citep{Hafner2020Dream}, and PC3~\citep{shu2020predictive}. The following two properties are extremely important in designing LCE models and algorithms. {\bf First}, to learn a representation that is the most suitable for the control process at hand. This suggests incorporating the control algorithm in the process of learning representation. This view of learning control-aware representations is aligned with the value-aware and policy-aware model learning, VAML~\citep{Farahmand18IV} and PAML~\citep{Abachi20PA}, frameworks that have been recently proposed in model-based RL. {\bf Second}, to interleave the representation learning and control processes, and to update them both, using a unifying objective function. This allows to have an end-to-end framework for representation learning and control.  

LCE methods, such as SOLAR and Dreamer, have taken steps towards the second objective by performing representation learning and control in an online fashion. This is in contrast to offline methods like E2C, RCE, PCC, and PC3, that learn a representation once and then use it in the entire control process. On the other hand, methods like PCC and PC3 address the first objective by adding a term to their representation learning loss function that accounts for the curvature of the latent dynamics. This term regularizes the representation towards smoother latent dynamics, which are suitable for the locally-linear controllers, e.g.,~iLQR~\citep{li2004iterative}, used by these methods. 

In this paper, we take a few steps towards the above two objectives. We first formulate a LCE model to learn representations that are suitable to be used by a policy iteration (PI) style algorithm in the latent space. We call this model {\em control-aware representation learning} (CARL). We derive a loss function for CARL that exhibits a close connection to the prediction, consistency, and curvature (PCC) principle for representation learning, proposed in~\cite{Levine2020Prediction}.  We derive three implementations of CARL: {\em offline}, {\em online}, and {\em value-guided}. Similar to offline LCE methods, such as E2C, RCE, PCC, and PC3, in {\em offline} CARL, we first learn a representation and then use it in the entire control process. However, in offline CARL, we replace the locally-linear control algorithm (e.g., iLQR) used by these LCE methods with a PI-style (actor-critic) RL algorithm. Our choice of RL algorithm is the model-based implementation of soft actor-critic (SAC)~\citep{pmlr-v80-haarnoja18b}. Our experiments show significant performance improvement by replacing iLQR with SAC. {\em Online} CARL is an iterative algorithm in which at each iteration, we first learn a latent representation by minimizing the CARL loss, and then perform several policy updates using SAC in this latent space. Our experiments with online CARL show further performance gain over its offline version. Finally, in {\em value-guided} CARL (V-CARL), we optimize a weighted version of the CARL loss function, in which the weights depend on the TD-error of the current policy. This would help to further incorporate the control algorithm in the representation learning process. We evaluate the proposed algorithms by extensive experiments on benchmark tasks and compare them with several LCE baselines.



\vspace{-0.1in}
\section{Problem Formulation}
\label{sec:prob-formulation}
\vspace{-0.1in}

We are interested in learning control policies for non-linear dynamical systems, where the states $s\in\mathcal S \subseteq \mathbb R^{n_s}$ are not fully observed and we only have access to their high-dimensional observations $x\in\mathcal X \subseteq \mathbb R^{n_x},\;n_x\gg n_s$. This scenario captures many practical applications in which we interact with a system only through high-dimensional sensory signals, such as image and audio. We assume that the observations $x$ have been selected such that we can model the system in the observation space using a Markov decision process (MDP)\footnote{A method to ensure observations are Markovian is to buffer them for several time steps~\citep{mnih2013playing}.} $\mathcal M_{\mathcal X}=\langle \mathcal X, \mathcal A, r, P,\gamma \rangle$, where $\mathcal X$ and $\mathcal A$ are observation and action spaces; $r:\mathcal X\times\mathcal A\rightarrow\mathbb R$ is the reward function with maximum value $R_{\max}$, defined by the designer of the system to achieve the control objective;\footnote{For example, in a goal tracking problem in which the agent (robot) aims at finding the shortest path to reach the observation goal $x_g$ (the observation corresponding to the goal state $s_g$), we may define the reward for each observation $x$ as the negative of its distance to $x_g$, i.e.,~$-\|x-x_g\|^2$.} $P:\mathcal X\times\mathcal A\rightarrow\mathbb P(\mathcal X)$ is the unknown transition kernel; and $\gamma\in (0,1)$ is the discount factor.  
Our goal is to find a mapping from observations to control signals, $\mu:\mathcal X\rightarrow\mathbb P(\mathcal A)$, with maximum expected return, i.e.,~$J(\mu) = \mathbb E[\sum_{t=0}^\infty \gamma^tr(x_t,a_t)\mid P,\mu]$. 


Since the observations $x$ are high-dimensional and the observation dynamics $P$ is unknown, solving the control problem in the observation space may not be efficient. As discussed in Section~\ref{sec:intro}, the class of {\em learning controllable embedding} (LCE) algorithms addresses this issue by learning a low-dimensional latent (embedding) space $\mathcal Z\subseteq\mathbb R^{n_z},\;n_z\ll n_x$ and a latent state dynamics, and controlling the system there. The main idea behind LCE is to learn an {\em encoder} $E:\mathcal X\rightarrow\mathbb P(\mathcal Z)$, a {\em latent space dynamics} $F:\mathcal Z\times\mathcal A\rightarrow\mathbb P(\mathcal Z)$, and a {\em decoder} $D:\mathcal Z\rightarrow\mathbb P(\mathcal X)$,\footnote{Some recent LCE models, such as PC3~\citep{shu2020predictive}, are advocating latent models without a decoder. Although we are aware of the merits of such approach, we use a decoder in the models proposed in this paper.} such that a good or optimal controller (policy) in $\mathcal Z$ minimizes the expected loss in the observation space $\mathcal X$. This means that if we model the control problem in $\mathcal Z$ as a MDP $\mathcal M_{\mathcal Z}=\langle \mathcal Z, \mathcal A, \bar{r}, F, \gamma \rangle$ and solve it using a model-based RL algorithm to obtain a policy $\pi:\mathcal Z\rightarrow\mathbb P(\mathcal A)$, the image of $\pi$ in the observation space, i.e.,~$(\pi\circ E)(a|x) = \int_zdE(z|x)\pi(a|z)$, should have a high return. Thus, the loss function to learn $\mathcal Z$ and $(E,F,D)$ from observations $\{(x_t,a_t,r_t,x_{t+1})\}$ should be designed to comply with this goal.

This is why in this paper, we propose a LCE framework that tries to incorporate the control algorithm used in the latent space in the representation learning process. We call this model, {\em control-aware representation learning} (CARL). In CARL, we set the class of control (RL) algorithms used in the latent space to approximate policy iteration (PI), and more specifically to soft actor-critic (SAC)~\citep{pmlr-v80-haarnoja18b}. Before describing CARL in details in the following sections, we present a number of useful definitions and notations here. 

For any policy $\mu$ in $\mathcal X$, we define its value function $U_\mu$ and Bellman operator $T_\mu$ as 

\vspace{-0.15in}
\begin{small}
\begin{equation}
\label{eq:value-Bellman-observation}
U_\mu(x) = \mathbb E[\sum_{t=0}^\infty\gamma^tr_\mu(x_t) \mid P_\mu,x_0=x], \qquad\qquad T_\mu[U](x) = \mathbb E_{x'\sim P_\mu(\cdot|x)}[r_\mu(x) + \gamma U(x')],
\end{equation}
\end{small}
\vspace{-0.15in}

for all $x\!\in\!\mathcal X$ and $U:\mathcal X\!\rightarrow\!\mathbb R$, where $r_\mu(x)\!=\!\int_a d\mu(a|x)r(x,a)$ and $P_\mu(x'|x)\!=\!\int_a d\mu(a|x)P(x'|x,a)$ are the reward function and dynamics induced by $\mu$. 

Similarly, for any policy $\pi$ in $\mathcal Z$, we define its induced reward function and dynamics as $\bar r_\pi(z)=\int_a d\pi(a|z)\bar r(z,a)$ and $F_\pi(z'|z)=\int_a d\pi(a|z)F(z'|z,a)$. We also define its value function $V_\pi$ and Bellman operator $T_\pi$, for all $z\in\mathcal Z$ and $V:\mathcal Z\rightarrow\mathbb R$, as 

\vspace{-0.15in}
\begin{small}
\begin{equation}
\label{eq:value-Bellman-latent}
V_\pi(z) = \mathbb E[\sum_{t=0}^\infty\gamma^t\bar{r}_\pi(z_t) \mid F_\pi,z_0=z], \qquad T_\pi[U](z) = \mathbb E_{z'\sim F_\pi(\cdot|z)}[{\bar r}_\pi(z) + \gamma V(z')].
\end{equation}
\end{small}
\vspace{-0.15in}


For any policy $\pi$ and value function $V$ in the latent space $\mathcal Z$, we denote by $\pi\circ E$ and $V\circ E$, their image in the observation space $\mathcal X$, given encoder $E$, and define them as

\vspace{-0.15in}
\begin{small}
\begin{equation}
\label{eq:policy-value-image}
(\pi\circ E)(a|x) = \int_z dE(z|x)\pi(a|z), \qquad\qquad (V\circ E)(x) = \int_z dE(z|x)V(z).
\end{equation}
\end{small}
\vspace{-0.15in}

\vspace{-0.1in}
\section{A Control Perspective for CARL Model}
\label{sec:CARL-Control}
\vspace{-0.1in}

In this section, we formulate our LCE model, which we refer to as {\em control-aware representation learning} (CARL). As described in Section~\ref{sec:prob-formulation}, CARL is a model for learning a low-dimensional latent space $\mathcal Z$ and the latent dynamics, from data generated in the observation space $\mathcal X$, such that this representation is suitable to be used by a policy iteration (PI) algorithm in $\mathcal Z$. In order to derive the loss function used by CARL to learn $\mathcal Z$ and its dynamics, i.e.,~$(E,F,D,\bar r)$, we first describe how the representation learning can be interleaved with PI in $\mathcal Z$. Algorithm~\ref{alg:PI-Latent} contains the pseudo-code of the resulting algorithm, which we refer to as {\em latent space learning policy iteration} (LSLPI). 

Each iteration $i$ of LSLPI starts with a policy $\mu^{(i)}$ in the observation space $\mathcal X$, which is the mapping of the improved policy in $\mathcal Z$ in iteration $i-1$, i.e.,~$\pi_+^{(i-1)}$, back in $\mathcal X$ through the encoder $E^{(i-1)}$ (Lines~\ref{line:LSLPI-actor} and~\ref{line:LSLPI-projX}). We then compute $\pi^{(i)}$, the current policy in $\mathcal Z$, as the image of $\mu^{(i)}$ in $\mathcal Z$ through the encoder $E^{(i)}$ (Line~\ref{line:LSLPI-distillation}). Note that $E^{(i)}$ is the encoder learned at the end of iteration $i-1$ (Line~\ref{line:LSLPI-rep-learn}). We then use the latent space dynamics $F^{(i)}$ learned at the end of iteration $i-1$ (Line~\ref{line:LSLPI-rep-learn}), and first compute the value function of $\pi^{(i)}$ in the policy evaluation or {\em critic} step, i.e.,~$V^{(i)}=V_{\pi^{(i)}}$ (Line~\ref{line:LSLPI-critic}), and then use $V^{(i)}$ to compute the improved policy $\pi_+^{(i)}$, as the greedy policy w.r.t.~$V^{(i)}$, i.e.,~$\pi^{(i+1)}=\mathcal G[V^{(i)}]$, in the policy improvement or {\em actor} step (Line~\ref{line:LSLPI-actor}). Using the samples in the buffer $\mathcal D$, together with the current policies in $\mathcal Z$, i.e.,~$\pi^{(i)}$ and $\pi_+^{(i)}$, we learn the new representation $(E^{(i+1)},F^{(i+1)},D^{(i+1)},\bar{r}^{(i+1)})$ (Line~\ref{line:LSLPI-rep-learn}). Finally, we generate samples $\mathcal D^{(i+1)}$ by following $\mu^{(i+1)}$, the image of the improved policy $\pi_+^{(i)}$ back in $\mathcal X$ using the old encoder $E^{i}$ (Line~\ref{line:LSLPI-projX}), and add it to the buffer $\mathcal D$ (Line~\ref{line:LSLPI-sample-gen}), and the algorithm iterates. It is important to note that both critic and actor operate in the low-dimensional latent space $\mathcal Z$. 


\begin{algorithm}[t]
\begin{small}
\caption{Latent Space Learning with Policy Iteration (LSLPI)}
\label{alg:PI-Latent}
\begin{algorithmic}[1]
\STATE {\bf Inputs}: $E^{(0)}$, $F^{(0)}$, $D^{(0)}$;
\STATE {\bf Initialization:} $\mu^{(0)}=$ random policy; $\quad$ $\mathcal D \leftarrow$ samples generated from $\mu^{(0)}$;
\FOR{$i = 0,1,\ldots$}
\STATE Compute $\pi^{(i)}$ such that $\mu^{(i)} = \pi^{(i)}\circ E^{(i)}$; \hfill {\color{gray}\# {\em set $\pi^{(i)}$ to be the image of $\mu^{(i)}$ in the latent space}} \label{line:LSLPI-distillation}
\STATE Compute the value function of $\pi^{(i)}$ and set $V^{(i)}= V_{\pi^{(i)}}$; \hfill {\color{gray}\# {\em policy evaluation (critic)}} \label{line:LSLPI-critic}
\STATE Compute the greedy policy w.r.t.~$V^{(i)}$ and set $\pi_+^{(i)}=\mathcal G[V^{(i)}]$; \hfill {\color{gray}\# {\em policy improvement (actor)}} \label{line:LSLPI-actor}
\STATE Set $\mu^{(i+1)} = \pi_+^{(i)} \circ E^{(i)}$; \hfill {\color{gray}\# {\em project the improved policy $\pi_+^{(i)}$ back into the observation space}} \label{line:LSLPI-projX}
\STATE Learn $(E^{(i+1)},F^{(i+1)},D^{(i+1)},{\bar r}^{(i+1)})$ from $\mathcal D$, $\pi^{(i)}$, and $\pi_+^{(i)}$; \hfill {\color{gray}\# {\em representation learning}} \label{line:LSLPI-rep-learn}
\STATE Generate samples $\mathcal D^{(i+1)}=\{(x_t,a_t,r_t,x'_t)\}_{t=1}^n$ from $\mu^{(i+1)}$; $\quad$ $\mathcal D \leftarrow \mathcal D \cup \mathcal D^{(i+1)}$; \label{line:LSLPI-sample-gen}
\ENDFOR
\end{algorithmic}
\end{small}
\end{algorithm}


LSLPI is a PI algorithm in $\mathcal Z$. However, what is desired is that it also acts as a PI algorithm in $\mathcal X$, i.e.,~it results in (monotonic) policy improvement in $\mathcal X$, i.e.,~$U_{\mu^{(i+1)}}\geq U_{\mu^{(i)}}$. Therefore, we define the representation learning loss function in CARL, such that it ensures that LSLPI also results in policy improvement in $\mathcal X$. The following theorem, whose proof is reported in Appendix~\ref{sec:API-CARL-Proofs}, shows the relationship between the value functions of two consecutive polices generated by LSLPI in $\mathcal X$.

\begin{theorem}
\label{thm:PI-Observation-Space}
Let $\mu$, $\mu_+$, $\pi$, $\pi_+$, and $(E,F,D,\bar r)$ be the policies $\mu^{(i)}$, $\mu^{(i+1)}$, $\pi^{(i)}$, $\pi_+^{(i)}$, and the learned latent representation $(E^{(i+1)},F^{(i+1)},D^{(i+1)},\bar r^{(i+1)})$ at iteration $i$ of the LSLPI algorithm (Algorithm~\ref{alg:PI-Latent}). Then, the following holds for the value functions of $\mu$ and $\mu_+$: 

\vspace{-0.175in}
\begin{small}
\begin{equation}
\label{eq:PI-observation}
U_{\mu_+}(x) \geq U_{\mu}(x) - \Big(\frac{\gamma}{1-\gamma}\sum_{\widetilde\pi\in\{\pi,\pi_+\}}\Delta(\enc,\dyn,\dec,\bar{r},\widetilde\pi,x) + \frac{R_{\max}}{1-\gamma}\cdot \underbrace{D_{\text{KL}}\big((\pi\circ E)(\cdot|x)\;||\;\mu(\cdot|x)\big)}_{\mathrm {L_{reg}(E,E_-,\pi,x)}}\Big),
\end{equation}
\end{small}
\vspace{-0.175in}

for all $x\in\mathcal X$, where the error term $\Delta$ for a policy $\pi$ is given by

\vspace{-0.175in}
\begin{small}
\begin{align}
\label{eq:error-term-PI-observation}
\Delta(&E,F,D,\bar{r},\pi,x) = \frac{R_{\max}}{1-\gamma}\overbrace{\sqrt{\frac{-1}{2}\int_zdE(z|x)\log D(x|z)}}^{(\mathrm {I})=\mathrm {L_{ed}(E,D,x)}} \;\; + \;\; 2\overbrace{\big|r_{\pi\circ E}(x) - \int_zdE(z|x)\bar r_\pi(z)\big|}^{(\mathrm {II})=\mathrm {L_r(E,\bar r,\pi,x)}} \\ 
&+ \frac{\gamma R_{\max}}{\sqrt{2}(1-\gamma)}\Big(\underbrace{\sqrt{D_{\text{KL}}\big(P_{\pi\circ E}(\cdot|x) \; || \; (D\circ F_\pi\circ E)(\cdot|x)\big)}}_{(\mathrm {III})=\mathrm {L_p(E,F,D,\pi,x)}} + \underbrace{\sqrt{D_{\text{KL}}\big((E \circ P_{\pi\circ E})(\cdot|x)\;||\;(F_\pi\circ E)(\cdot|x)\big)}}_{(\mathrm {IV})}\Big). \nonumber
\end{align}
\end{small}
\vspace{-0.175in}
\end{theorem}

It is easy to see that LSLPI guarantees (policy) improvement in $\mathcal X$, if the terms in the parentheses on the RHS of (\ref{eq:PI-observation}) are zero. We now describe these terms. The last term on the RHS of (\ref{eq:PI-observation}) is the KL between $\pi^{(i)} \circ E$ and $\mu^{(i)}=\pi^{(i)}\circ E^{(i)}$. This term can be seen as a regularizer to keep the new encoder $E$ close to the old one $E^{(i)}$. The four terms in (\ref{eq:error-term-PI-observation}) are: {\bf (I)} The encoding-decoding error to ensure $x\approx (D\circ E)(x)$; {\bf (II)} The error that measures the mismatch between the reward of taking action according to policy $\pi\circ E$ at $x\in\mathcal X$, and the reward of taking action according to policy $\pi$ at the image of $x$ in $\mathcal Z$ under $E$; {\bf (III)} The error in predicting the next observation through paths in $\mathcal X$ and $\mathcal Z$. This is the error between $x'$ and $\hat{x}'$ shown in Fig.~\ref{fig:Prediction-Consistency}(a); and {\bf (IV)} The error in predicting the next latent state through paths in $\mathcal X$ and $\mathcal Z$. This is the error between $z'$ and $\tilde{z}'$ shown in Fig.~\ref{fig:Prediction-Consistency}(b). 


\vspace{-0.05in}
\begin{figure}[h]
\centering
\includegraphics[width=2in]{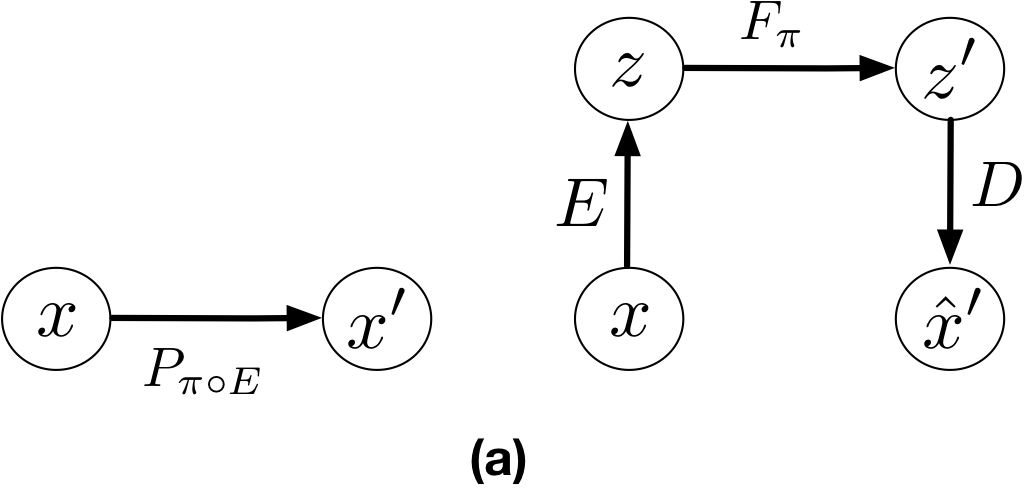}
\hspace{0.75in}
\includegraphics[width=2in]{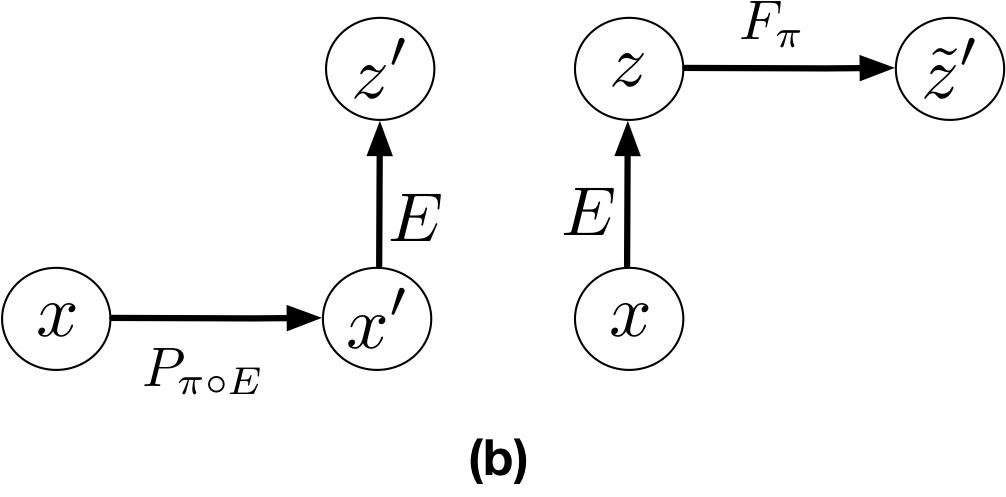}
\vspace{-0.05in}
\caption{{\bf (a)} Paths from the current observation $x$ to the next one, (left) in $\mathcal X$ and (right) through $\mathcal Z$. {\bf (b)} Paths from the current observation $x$ to the next latent state, (left) through $\mathcal X$ followed by encoding and (right) starting with encoding and through $\mathcal Z$.}
\label{fig:Prediction-Consistency}
\vspace{-0.05in}
\end{figure}

{\bf Representation Learning in CARL:} $\;$ Theorem~\ref{thm:PI-Observation-Space} provides us with a recipe (loss function) to learn the latent space $\mathcal Z$ and $(E,F,D,\bar r)$. In CARL, we propose to learn a representation for which the terms in the parentheses on the RHS of (\ref{eq:PI-observation}) are small. 
As mentioned earlier, the second term, $L_{\text{reg}}(E,E_-,\pi,x)$, can be considered as a regularizer to keep the new encoder $E$ close to the old one $E_-$. Term (II) that measures the mismatch between rewards can be kept small, or even zero, if the designer of the system selects the rewards in a compatible way. Although CARL allows us to learn a reward function in the latent space, similar to several other LCE works~\citep{E2C,Banijamali2018RobustLC,Levine2020Prediction,shu2020predictive}, in this paper, we assume that a compatible reward function is given. Terms (III) and (IV) are the equivalent of the {\em prediction} and {\em consistency} terms in PCC~\citep{Levine2020Prediction} for a particular latent space policy $\pi$. Since PCC has been designed for an offline setting (i.e.,~one-shot representation learning and control), the prediction and consistency terms are independent of a particular policy and are defined for state-action pairs. While CARL is designed for an online setting (i.e.,~interleaving representation learning and control), and thus, its loss function at each iteration depends on the current latent space policies $\pi$ and $\pi_+$. As we will see in Section~\ref{subsec:offline-CARL}, our offline implementation of CARL uses a loss function very similar to that of PCC. Note that (IV) is slightly different than the consistency term in PCC. However, if we upper-bound it using Jensen inequality as \begin{small}$(\mathrm {IV}) \leq L_c(E,F,\pi,x):= \int_{x'\in\mathcal X}dP_{\pi\circ E}(x'|x)\cdot D_\text{KL}\big( E(\cdot|x')\;||\;(F_\pi\circ E)(\cdot|x)\big)$\end{small},
%
%
we obtain the loss term $L_{\text{c}}(E,F,\pi,x)$, which is very similar to the consistency term in PCC. Similar to PCC, we also add a curvature loss to the loss function of CARL to encourage having a smoother latent space dynamics $F_\pi$. Putting all these terms together, we obtain the CARL loss function as 

\vspace{-0.175in}
\begin{small}
\begin{align}
(E^*,F^*,D^*) \in \argmin_{(E,F,D)}\sum_{x\sim \mathcal D}& \lambda_{\text{ed}} L_{\text{ed}}(E,D,x) + \lambda_{\text{p}} L_{\text{p}}(E,F,D,\pi,x) 
+ \lambda_{\text{c}} L_{\text{c}}(E,F,\pi,x) \nonumber \\ 
&+ \lambda_{\text{cur}}L_{\text{cur}}(F,\pi,x) + \lambda_{\text{reg}} L_{\text{reg}}(E,E_-,\pi,x),
\label{eq:CARL-loss}
\end{align}
\end{small}
\vspace{-0.185in}

where $(\lambda_{\text{ed}},\lambda_{\text{p}},\lambda_{\text{c}},\lambda_{\text{cur}},\lambda_{\text{reg}})$ are hyper-parameters of the algorithm, $(L_{\text{ed}},L_{\text{p}})$ are the encoding-decoding and prediction losses defined in (\ref{eq:error-term-PI-observation}), $L_{\text{c}}$ is the consistency loss defined above, $L_{\text{cur}}$ is the curvature loss, and $L_{\text{reg}}$ is the regularizer that ensures the new encoder remains close to the old one. We set $\lambda_{\text{reg}}$ to a small value not to allow $L_{\text{reg}}$ to play a major role in our implementations. 


\vspace{-0.1in}
\section{Different Implementations of CARL}
\label{sec:CARL-algos}
\vspace{-0.1in}

The CARL loss function in (\ref{eq:CARL-loss}) introduces an optimization problem that takes a policy $\pi$ as input and learns a representation suitable for its evaluation and improvement. To optimize this loss in practice, similar to the PCC model~\citep{Levine2020Prediction}, we define $\widehat{P}=D\circ F_\pi\circ E$ as a latent variable model that factorizes as $\widehat{P}(x_{t+1},z_t,\hat{z}_{t+1}|x_t,\pi)=\widehat{P}(z_t|x_t)\widehat{P}(\hat{z}_{t+1}|z_t,\pi)\widehat{P}(x_{t+1}|\hat{z}_{t+1})$, and use a variational approximation to the interactable negative log-likelihood of the loss terms in (\ref{eq:CARL-loss}). The variational bounds for these terms can be obtained similar to Eqs.~6 and~7 in~\cite{Levine2020Prediction}. Below we describe three instantiations of the CARL model in practice. Implementation details can be found in Algorithm~\ref{alg:CARL_implementation} in Appendix~\ref{app:CARL Algorithm}. While CARL is compatible with most PI-style (actor-critic) RL algorithms, following a recent work, MBRL~\citep{MBPO2019Janner}, we choose SAC as the RL algorithm in CARL. Since most actor-critic algorithms are based on first-order gradient updates, as discussed in Section~\ref{sec:CARL-Control}, we regularize the curvature of the latent dynamics $F$ (see Eqs.~8 and~9 in~\citealt{Levine2020Prediction}) in CARL to improve its empirical stability and performance in policy learning.

\vspace{-0.075in}
\subsection{Offline CARL}
\label{subsec:offline-CARL}
\vspace{-0.075in}

We first implement CARL in an offline setting, where we first generate a (relatively) large batch of observation samples $\{(x_t,a_t,r_t,x'_t)\}_{t=1}^N$ using an exploratory (e.g.,~random) policy. We then use this batch to optimize the CARL loss function (\ref{eq:CARL-loss}) via a variational approximation scheme, as described above, and learn a latent representation $\mathcal Z$ and $(E,F,D)$. Finally, we solve the decision problem in $\mathcal Z$ using a model-based RL algorithm, which in our case is model-based soft actor-critic (SAC)~\citep{pmlr-v80-haarnoja18b}. The learned policy $\hat{\pi}^*$ in $\mathcal Z$ is then used to control the system from observations as $a_t\sim(\hat{\pi}^*\circ E)(\cdot|x_t)$. This is the setting that has been used in several recent LCE works, such as E2C~\citep{E2C}, RCE~\citep{Banijamali2018RobustLC}, PCC~\citep{Levine2020Prediction}, and PC3~\citep{shu2020predictive}. Our offline implementation is different than these works in {\bf 1)} we replace the locally linear control algorithm, namely iterative LQR (iLQR)~\citep{li2004iterative}, used in them with model-based SAC, which results in significant performance improvement as shown in our experimental results in Section~\ref{sec:experiments}, and {\bf 2)} we optimize the CARL loss function that as mentioned above, despite close connection is still different than the one used by PCC. 

The CARL loss function presented in Section~\ref{sec:CARL-Control} has been designed for an online setting in which at each iteration, it takes a policy as input and learns a representation that is suitable for evaluating and improving this policy. However, in the offline setting, the learned representation should be good for any policy generated in the course of running the PI-style control algorithm. Therefore, we marginalize out the policy from the (online) CARL's loss function and use the RHS of the following corollary (proof in Appendix~\ref{app:offline_CARL}) to construct the CARL loss function used in our offline experiments.

\begin{corollary}
Let $\mu$ and $\mu_+$ be two consecutive policies in $\mathcal X$ genrated by a PI-style control algorithm in the latent space constructed by $(E,\!F,\!D,\!\bar r)$. Then, the following holds for the value functions of $\mu$ and $\mu_+$, where $\Delta$ is defined by (\ref{eq:error-term-PI-observation}) (in modulo to replacing sampled action $a\!\sim\!\pi\!\circ\!\enc$ with action $a$):

\vspace{-0.15in}
\begin{small}
\begin{equation}
\label{eq:policy_improve_offline}
U_{\mu_+}(x) \geq U_\mu(x)-\frac{2\gamma}{1-\gamma} \cdot \max_{a\in\mathcal A}\;\Delta(\enc,\dyn,\dec,\bar{r},a,x),\,\,\forall x\in\mathcal X.
\end{equation}
\end{small}
\vspace{-0.275in}
\end{corollary}

\vspace{-0.05in}
\subsection{Online CARL}
\label{subsec:online-CARL}
\vspace{-0.05in}

In the online implementation of CARL, at each iteration $i$, the current policy $\pi^{(i)}$ is the improved policy of the last iteration, $\pi_+^{(i-1)}$. We first generate a relatively (to offline CARL) small batch samples from the image of this policy in $\mathcal X$, i.e.,~$\mu^{(i)}=\pi^{(i)}\circ E^{(i-1)}$, and then learn a representation $(E^{(i)},F^{(i)},D^{(i)})$ suitable for evaluating and improving the image of $\mu^{(i)}$ in $\mathcal Z$ under the new encoder $E^{(i)}$. This means that with the new representation, the current policy that was the image of $\mu^{(i)}$ in $\mathcal Z$ under $E^{(i-1)}$, should be replaced by its image $\pi^{(i)}$ under the new encoder, i.e.,~$\pi^{(i)}\!\circ\! E^{(i)}\approx \mu^{(i)}$. In online CARL, this is done by a {\em policy distillation} step in which we minimize the following loss:\footnote{Our experiments reported in Appendix~\ref{app:Additional Experiments} show that adding distillation improves the performance in online CARL. Thus, all the results reported for online CARL and value-guided CARL in the main paper are with policy distillation.}

\vspace{-0.125in}
\begin{small}
\begin{equation}
\label{eq:policy-distillation}
\pi^{(i)} \in \argmin_\pi \sum_{x\sim\mathcal D} D_{\text{KL}}\big((\pi\circ E^{(i)})(\cdot|x) \;||\; (\pi_+^{(i-1)}\circ E^{(i-1)})(\cdot|x)\big).
\end{equation}
\end{small}
\vspace{-0.125in}

After the current policy $\pi^{(i)}$ was set, we perform multiple steps of (model-based) SAC in $\mathcal Z$ using the current model, $(F^{(i)},\bar r^{(i)})$, and then send the resulting policy $\pi_+^{(i)}$ to the next iteration. 



\vspace{-0.05in}
\subsection{Value-Guided CARL}
\label{subsec:V-CARL}
\vspace{-0.05in}

In the previous two implementations of CARL, we learn the model $(E,F,D)$ using the loss function~(\ref{eq:CARL-loss}). Theorem~\ref{thm:PI-Observation-Space} shows that minimizing this loss guarantees performance improvement. While this loss depends on the current policy $\mu$ (through the latent space policy and encoder), it does not use its value function $U_\mu$. To incorporate this extra piece of information in the representation learning process, we utilize results from variational model-based policy optimization (VMBPO) work by~\cite{chow2020Variational}. Using Lemma~3 in~\cite{chow2020Variational}, we can include the value function in the observation space dynamics as\footnote{In general, this can also be applied to reward learning, but for simplicity we only focus on learning dynamics.}

\vspace{-0.275in}
\begin{small}
\begin{equation}
\label{eq:optimal-posteriors-closed-form}
\begin{split}
P^*(x'|x,a) = P(x'|x,a)\cdot\exp\big(\tau\cdot\frac{r(x,a)+ \gamma\tilde{U}_\mu(x')-\tilde{W}_\mu(x,a) }{\gamma}\big),
\end{split}
\end{equation}
\end{small}
\vspace{-0.15in}


where \begin{small}$\tilde{U}_\mu(x):=\frac{1}{\tau}\log\mathbb E\left[\exp\left(\tau\cdot\sum_{t=0}^\infty\gamma^tr_{\mu,t} \right)\mid P_\mu,x_0=x\right]$\end{small} is the \emph{risk-adjusted} value function of policy $\mu$, and $\tilde{W}_\mu(x)$ is its corresponding state-action value function, i.e., \begin{small}$\tilde{W}_\mu(x,a):=r(x,a)+\gamma\cdot\frac{1}{\tau}\log\mathbb E_{x'\sim P(\cdot|x,a)}\left[\exp(\tau\cdot U_\mu(x'))\right]$\end{small}. The reason for referring to $\tilde{W}_\mu$ and $\tilde{U}_\mu$ as risk-adjusted value functions is that the Bellman operator is no longer defined by the expectation operator over $P(x'|x,a)$, but instead is defined by the exponential risk \begin{small}$\rho_\tau(U(\cdot)|x,a)=\frac{1}{\tau}\log\mathbb E_{x'\sim P(\cdot|x,a)}[\exp(\tau\cdot U(x'))]$\end{small}, with the temperature parameter $\tau>0$~\citep{ruszczynski2006optimization}. The modified dynamics $P^*$ is a re-weighting of $P$ using the exponential-twisting weights \begin{small}$\exp(\frac{\tau}{\gamma}\cdot w(x,a,x'))$, where $w(x,a,x') := r(x,a)+\gamma \tilde{U}_\mu(x')-\tilde{W}_\mu(x,a)$\end{small} is the temporal difference (TD) error of the risk-adjusted value functions.

To incorporate the value-guided transition model $P^*$ into the CARL loss function, we need to modify all the loss terms that depend on $P$, i.e.,~the prediction loss $L_{\text{p}}(E,F,D,\pi,x)$ and the consistency loss $L_{\text{c}}(E,F,\pi,x)$. Because of the regularization term $L_{\text{reg}}(E,E_-,\pi,x)$ that enforces the policy $\pi\circ\enc$ to be close to $\mu$, we may replace the transition dynamics $P_{\pi\circ\enc}$ in the prediction loss $L_{\text{p}}(E,F,D,\pi,x)$ with $P_{\mu}$. Since $\log P_{\mu}(x'|x)$ does not depend on the representation, minimizing the prediction loss would be equivalent to maximizing the expected log-likelihood (MLE) $L_{\text{mle}}(E,F,D,\pi,x)=-\int_{x'}dP_{\mu}(x'|x)\cdot\log (D\circ F_\pi\circ E)(x'|x)$. Now if we replace dynamics $P$ with its value-guided counterpart $P^*$ in the MLE loss, we obtain the modified prediction loss

 
\vspace{-0.175in}
\begin{small}
\begin{equation}
\begin{split}
L^\prime_{\text{p}}(E,F,&\,D,\pi,x)=-\int_{a}d\mu(a|x)\int_{x'} dP(x'|x,a) \cdot \exp(\frac{\tau}{\gamma}\cdot w(x,a,x'))\cdot\log (D\circ F_\pi\circ E)(x'|x),
\end{split}
\end{equation}
\end{small}
\vspace{-0.2in}
  
which corresponds to a weighted MLE loss in which the weights are the exponential TD errors. 

Using analogous arguments, we may write the value-guided version of the consistency loss $L_{\text{c}}(E,F,\pi,x)=\int_{x'\in\mathcal X}dP_{\mu}(x'|x)\cdot D_\text{KL}(E(\cdot|x')||(F_\pi\circ E)(\cdot|x))$ as
 
\vspace{-0.15in}
\begin{small}
\begin{equation}
\begin{split}
L'_{\text{c}}(E,F,&D,\pi,x) =\int_{a}d\mu(a|x)\int_{x'} dP(x'|x,a) \cdot\exp(\frac{\tau}{\gamma}\cdot w(x,a,x'))\cdot D_\text{KL}(E(\cdot|x')||(F_\pi\circ E)(\cdot|x)).
\end{split}
\end{equation}
\end{small}
\vspace{-0.15in}
 
To complete the value-guided CARL (V-CARL) implementation, we need to compute the risk-adjusted value functions $\tilde{U}_\mu$ and $\tilde{W}_\mu$ to construct the weight $w(x,a,x')$. Here we provide a recipe for the case when $\tau$ is small (see Appendix~\ref{app:Value-guided CARL} for details), in which the risk-adjusted value functions can be approximated by their risk-neutral counterparts, i.e.,~$\tilde{U}_\mu(x)\!\approx\! U_\mu(x)$, and $\tilde{W}_\mu(x,a)\!\approx\! W_\mu(x,a)\!:=\!r(x,a)+\int_{x'}\!dP(x'|x,a)U_\mu(x')$. Following Lemma~\ref{lem:policy_eval} in Appendix~\ref{app:policy_eval}, we can approximate $U_{\mu}(x)$ with $(V_\pi\circ E)(x)$ and $W_\mu(x,a)$ with $(Q_\pi\circ E)(x,a)$. Together with the compatibility of the rewards, i.e., $r(x,a)\approx (\bar{r}\circ E)(x,a)$, 
the weight $w(x,a,x')$ can be approximated by $\begin{small}\widehat{w}(x,a,x'):=\int_{z,z'}dE(z|x)\cdot dE(z'|x')\cdot(\bar{r}(z,a)-Q_\pi(z,a)+\gamma V_\pi(z'))\end{small}$, which is simply the TD error in the latent space. 

\vspace{-0.1in}
\section{Experimental Results}
\label{sec:experiments}
\vspace{-0.1in}


In this section, we experiment with the following continuous control domains (see Appendix \ref{app:Experimental Details} for more descriptions): (i) Planar System, (ii) Inverted Pendulum, (iii) Cartpole, (iv) 3-link manipulator, and compare the performance of our CARL algorithms with two LCE baselines: PCC \citep{Levine2020Prediction} and SOLAR \citep{zhang19m}. These tasks have underlying start and goal states that are ``not'' observable to the algorithms, instead the algorithms only have access to the start and goal image observations. To evaluate the performance of the control algorithms, similar to \cite{Levine2020Prediction}, we report the $\%$-time spent in the goal. The initial policy that is used for data generation is uniformly random. To measure performance reproducibility for each experiment, we 1) train $25$ models and 2)  perform $10$ control tasks for each model. For SOLAR, due to its high computation cost we only train and evaluate $10$ different models. Besides the average results, we also report the results from the best LCE models, averaged over the $10$ control tasks.

\vspace{-0.05in}
\begin{paragraph}{General Results}
Table~\ref{tab:main_results} shows the mean and standard deviations of the metrics on different control tasks. To compare the data efficiency of different methods, we also report the number of samples required to train the latent space and controller in each method. Our main experimental observation is four-fold. First, by replacing the iLQR control algorithm used by PCC with model-based SAC, offline CARL achieves significantly better performance in all of the tasks. This can be attributed to the advantage that SAC is more robust and effective operating in non-(locally)-linear environments. More details in the comparisons between PCC and offline CARL can be found in Appendix \ref{app:Additional Experiments}, in which we explicitly compare their control performance and latent representation maps. Second, in all the tasks online CARL is more data-efficient than the offline counterparts, i.e., we can achieve similar or better performance with fewer environment samples. In particular, online CARL is notably superior in the planar, cartpole, and swingup tasks, in which similar performance can be achieved with $2$, $2.5$, and $4$ times less samples, respectively (see Figure \ref{fig:uniform_improvement}). From Figure \ref{fig:latent_space_improvement}, interestingly the latent representations of online CARL also tend to improve progressively with the value of the policy.  
Third, in the simpler tasks (Planar, Swingup, Cartpole), value-guided CARL (V-CARL) manages to achieve even better performance. This corroborates with our hypothesis that extra improvement can be delivered by CARL when its LCE model is more accurate in regions of the $\mathcal Z$ space with higher temporal difference --- regions with higher anticipated future return. Unfortunately, in three-pole, the performance of V-CARL is worse than its online counterpart. This is likely due to instability in representation learning when the sample variance is amplified by the exponential-TD weight. Finally, SOLAR requires much more samples to learn a reasonable latent space for control, and with limited data it fails to converge to a good policy. Note that we are able to obtain better results in the planar task when the goal location is fixed (see Appendix \ref{sec:Additional SOLAR Results}). Yet even with the fine-tuned latent space from~\cite{zhang19m}, its performance is incomparable with that of CARL algorithms.
\end{paragraph}




\begin{table}[h]
\begin{small}
\centering
\begin{tabular}{|l | l | l | l | l|}
\hline
\textbf{Environment} & \textbf{Algorithm} & \textbf{Number of
Samples} & \textbf{Avg $\%$-Goal} & \textbf{Best $\%$-Goal} \\
\hline
Planar & PCC & 5000 & $38.85 \pm 2.45$ & $62.5\pm 10.42$\\
\hline
Planar & Offline CARL & 5000 & $63.43\pm2.78$ & $\textbf{79.51} \pm \textbf{0.38}$\\
\hline
Planar & Online CARL & 3072 & $68.03 \pm 1.69$ & $79.02 \pm 0.38$\\
\hline
Planar & Value-Guided CARL & 3200 & $\textbf{71.05} \pm \textbf{1.46}$ & $\textbf{79.51}\pm \textbf{0.38}$\\
\hline
Planar & SOLAR & 5000 (VAE) + 16000 (Control) & $5.82\pm 2.50$ & $9.13 \pm 3.54$\\

\hline\hline
Swingup & PCC & 5000 & $86.60 \pm  1.00$ & $97.40\pm 0.61$\\
\hline
Swingup & Offline CARL & 5000 & $88.43\pm 2.02$ & $\textbf{98.50} \pm \textbf{0.0}$\\
\hline
Swingup & Online CARL & 1408 & $95.04\pm 0.96$ & $\textbf{98.50} \pm \textbf{0.0}$ \\
\hline
Swingup & Value-Guided CARL & 1408 & $\textbf{96.50} \pm \textbf{0.25}$ & $\textbf{98.50}\pm \textbf{0.0}$\\
\hline 
Swingup & SOLAR & 5200 (VAE) + 40000 (Control) & $16.1 \pm 0.69$& $22.45\pm 1.96$ \\
\hline\hline
Cartpole & PCC & 10000 & $83.64\pm 0.63$ & $\textbf{100.0} \pm \textbf{0.0}$\\
\hline
Cartpole & Offline CARL & 10000 & $91.11 \pm 1.50$ & $\textbf{100.0} \pm \textbf{0.0}$ \\
\hline
Cartpole & Online CARL & 5120 & $95.34 \pm 1.17$ & $\textbf{100.0} \pm \textbf{0.0}$\\
\hline
Cartpole & Value-Guided CARL & 5120 & $\textbf{95.79} \pm \textbf{1.06}$ & $\textbf{100.0} \pm \textbf{0.0}$\\
\hline
Cartpole & SOLAR & 5000 (VAE) + 40000 (Control) & $10.61 \pm 2.58$ & $12.33\pm 2.96$\\
\hline\hline
Three-pole & PCC & 4096 & $4.41 \pm 0.75$ & $36.20 \pm 7.06$ \\
\hline
Three-pole & Offline CARL & 4096 & $\textbf{63.20}\pm \textbf{1.77}$ & $88.55\pm 0.0$\\
\hline
Three-pole & Online CARL & 2944 & $62.17\pm2.28$ & $\textbf{90.05} \pm \textbf{0.0}$ \\
\hline
Three-pole & Value-Guided CARL & 2816 & $55.06\pm2.42$ & $89.05\pm 0.0$\\
\hline
Three-pole & SOLAR & 2000 (VAE) + 20000 (Control) & $0\pm 0$ & $0\pm 0$\\
\hline
\end{tabular}
\caption{ Mean $\pm$ standard error results ($\%$-goal) and samples used for different LCE algorithms. }
\label{tab:main_results}
\end{small}
\vspace{-0.25in}
\end{table}

\begin{figure}
    \centering
    \begin{subfigure}[b]{0.24\textwidth}
        \includegraphics[width=1.25in]{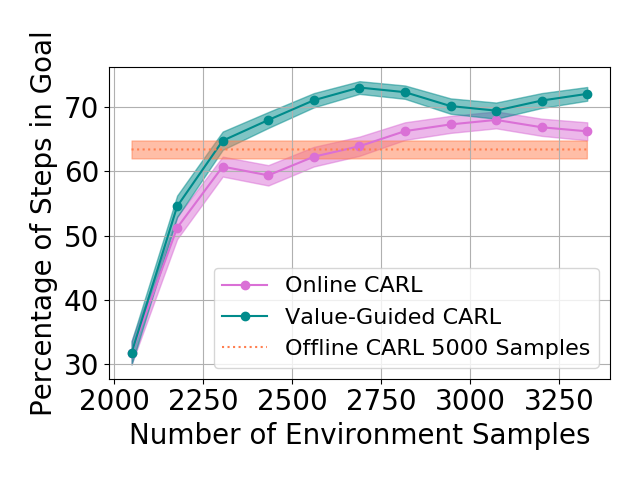}
        \caption{Planar}
    \end{subfigure}
    \begin{subfigure}[b]{0.24\textwidth}
        \includegraphics[width=1.25in]{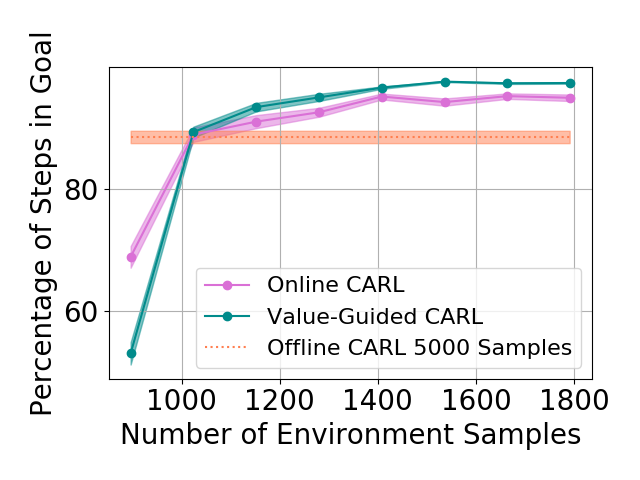}
        \caption{Swingup}
    \end{subfigure}
    \begin{subfigure}[b]{0.24\textwidth}
        \includegraphics[width=1.25in]{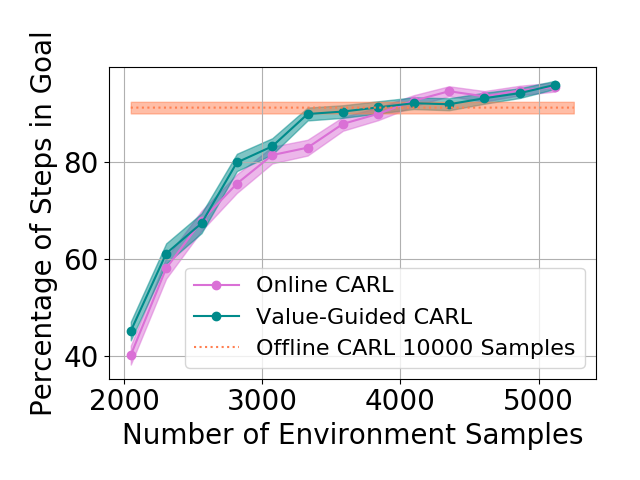}
        \caption{Cartpole}
    \end{subfigure}
    \begin{subfigure}[b]{0.24\textwidth}
        \includegraphics[width=1.25in]{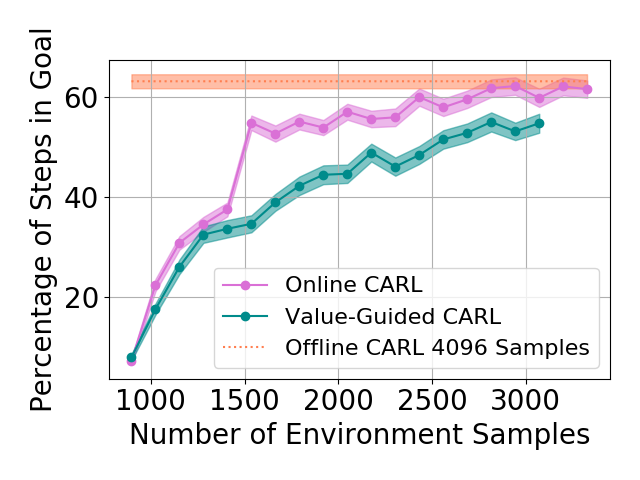}
        \caption{Three-pole}
    \end{subfigure}
    \vspace{-0.05in}
    \caption{Training curves of offline CARL, online CARL, and value-guided CARL. The shaded region represents mean $\pm$ standard error. Online variants of CARL is more data-efficient than offline CARL.}
    \label{fig:uniform_improvement}
    \vspace{-0.2in}
\end{figure}

\begin{figure}
    \centering
    \begin{subfigure}[b]{0.19\textwidth}
        \includegraphics[width=1.05in]{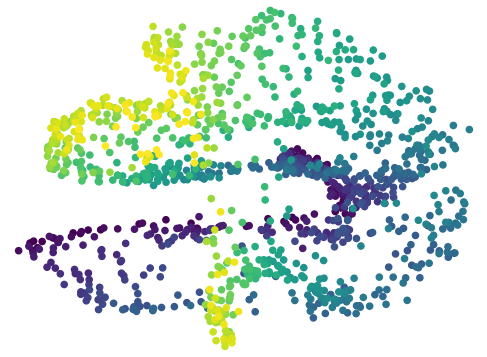}
        \caption{$i=1$}
    \end{subfigure}
    \begin{subfigure}[b]{0.19\textwidth}
        \includegraphics[width=1.05in]{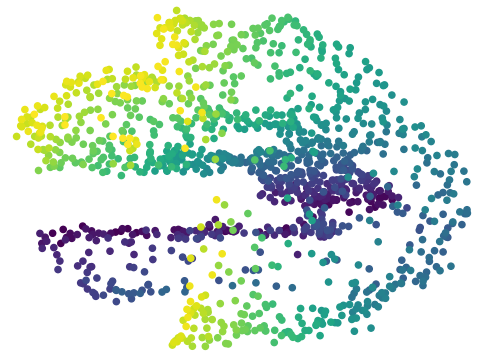}
        \caption{$i=3$}
    \end{subfigure}
    \begin{subfigure}[b]{0.19\textwidth}
        \includegraphics[width=1.05in]{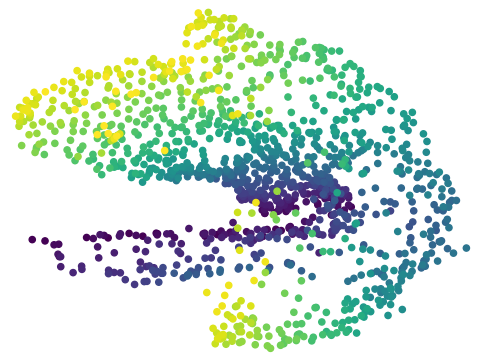}
        \caption{$i=6$}
    \end{subfigure}
    \begin{subfigure}[b]{0.19\textwidth}
        \includegraphics[width=1.05in]{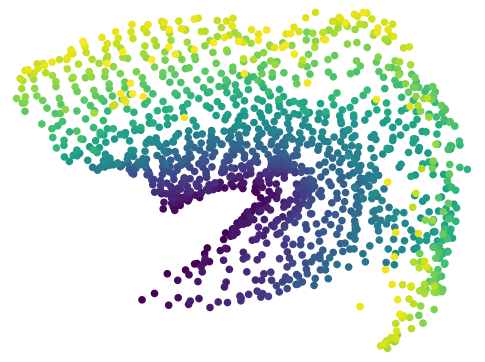}
        \caption{$i=10$}
    \end{subfigure}
    \begin{subfigure}[b]{0.19\textwidth}
        \includegraphics[width=1.05in]{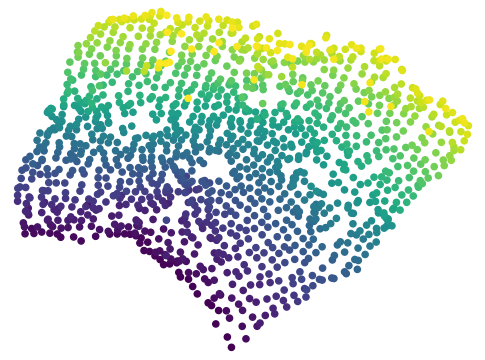}
        \caption{$i=14$}
    \end{subfigure}
    \vspace{-0.05in}
    \caption{Evolution of the latent representation of the Planar problem learned by online CARL. Here $i$ represents the number of LCE model-learning episodes (Algorithm \ref{alg:CARL_implementation} in Appendix \ref{app:CARL Algorithm})}
    \label{fig:latent_space_improvement}
    \vspace{-0.15in}
\end{figure}

\begin{figure}[th!]
    \centering
    \begin{subfigure}[b]{0.24\textwidth}
        \includegraphics[width=1.25in]{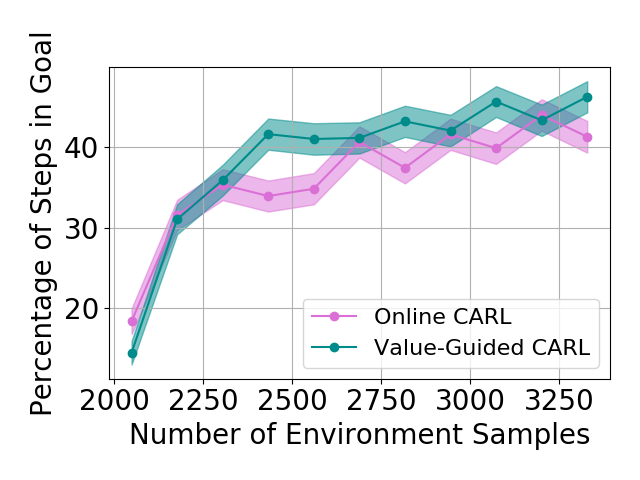}
        \caption{Planar}
    \end{subfigure}
    \begin{subfigure}[b]{0.24\textwidth}
        \includegraphics[width=1.25in]{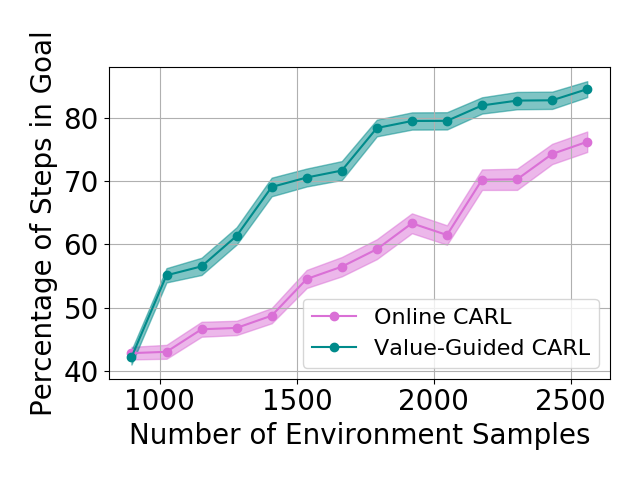}
        \caption{Swingup}
    \end{subfigure}
    \begin{subfigure}[b]{0.24\textwidth}
        \includegraphics[width=1.25in]{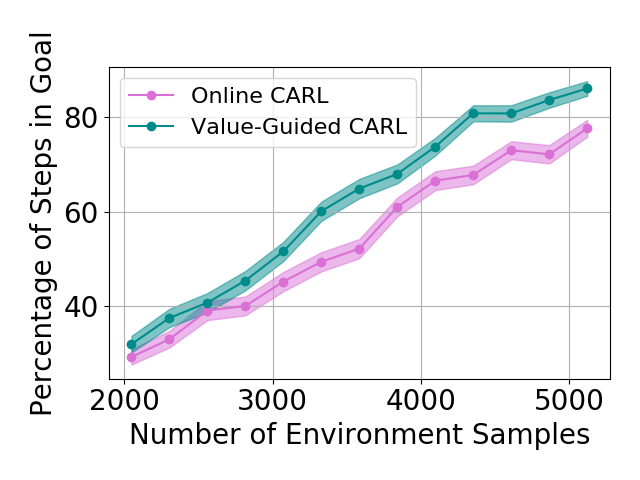}
        \caption{Cartpole}
    \end{subfigure}
    \begin{subfigure}[b]{0.24\textwidth}
        \includegraphics[width=1.25in]{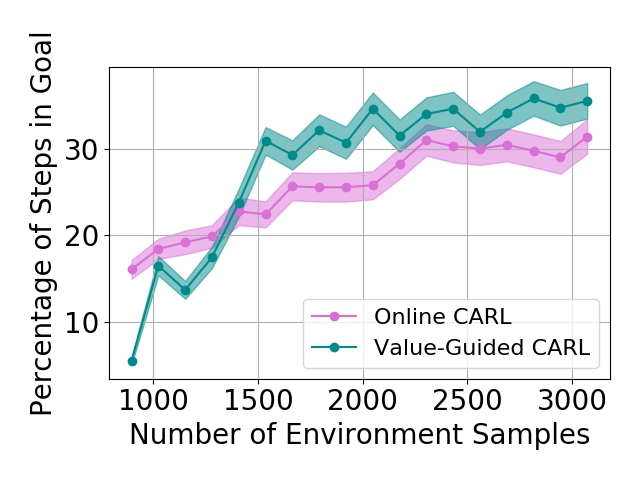}
        \caption{Three-pole}
    \end{subfigure}
    \vspace{-0.05in}
    \caption{Training curves comparing Online CARL and Value-Guided CARL when the initial samples are from the biased regions as described in Appendix \ref{app:Experimental Details}. }
    \label{fig:VMBPO}
    \vspace{-0.15in}
\end{figure}

\begin{paragraph}{Results with Environment-biased Sampling} In the previous experiments, all the online LCE algorithms are warm-started with data collected by a uniformly random policy over the entire environment. With this uniform data collection, we do not observe a significant difference between online CARL and V-CARL. This is because with sufficient data the LCE dynamics is accurate enough on most parts of the state-action space for control. To further illustrate the advantage of V-CARL over online CARL, we modify the experimental setting by gathering initial samples only from a specific region of the environment (see Appendix \ref{app:Experimental Details} for details). Figure~\ref{fig:VMBPO} shows the learning curves of online CARL and V-CARL in this case. Clearly with biased initial data both algorithms experience a certain level of performance degradation. Yet, V-CARL clearly outperforms online CARL. This again verifies our conjecture that value-aware LCE models are more robust to initial data distribution and more superior in policy optimization.
\end{paragraph}

\vspace{-0.1in}
\section{Conclusions}
\label{sec:conclu}
\vspace{-0.1in}

In this paper, 
we argued for incorporating control in the representation learning process and for the interaction between control and representation learning in learning controllable embedding (LCE) algorithms. We proposed a LCE model called {\em control-aware representation learning} (CARL) that learns representations suitable for policy iteration (PI) style control algorithms. We proposed three implementations of CARL that combine representation learning with model-based soft actor-critic (SAC), as the controller, in offline and online fashions. In the third implementation, called {\em value-guided} CARL, we further included the control process in representation learning by optimizing a weighted version of the CARL loss function, in which the weights depend on the TD-error of the current policy. We evaluated the proposed algorithms on benchmark tasks and compared them with several LCE baselines. The experiments show the importance of SAC as the controller and of the online implementation. Future directions include {\bf 1)} investigating other PI-style algorithms in place of SAC, {\bf 2)} developing LCE models suitable for value iteration style algorithms, and {\bf 3)} identifying other forms of bias for learning an effective embedding and latent dynamics.



\newpage
\section*{Broader Impact}

Controlling non-linear dynamical systems from high-dimensional observation is challenging. Direct application of model-free and model-based reinforcement learning algorithms to this problem may not be efficient, due to requiring a large number of samples from the real system (model-free) and the challenges of learning a model in a high-dimensional space (model-based). A popular approach to address this problem is {\em learning controllable embedding} (LCE), i.e.,~learning a low-dimensional latent space and a latent space model, and performing the optimal control in this learned latent space. This work is a step towards end-to-end representation learning and control in this setting. We propose methods that interleave representation learning and control, which allows us to learn control-aware representations, i.e.,~representations that are suitable for the control problem at hand. 


\bibliographystyle{plainnat}
\bibliography{Control-Aware-PCC}


\newpage
\appendix


\section{Approximate Latent Policy Evaluation}
\label{app:policy_eval}
To start with, for any value function $U:\mathcal X\rightarrow\reals$, at any observation $x\in\mathcal X$ and any arbitrary policy $\mu$ the $\mu$-induced observation-space Bellman operator can be written as:
\begin{small}
 \[
T_\mu[U](x):=\int_{a\in\mathcal A} d\mu(a|x)\int_{x'\in\mathcal X}dP(x'|x,a)\cdot\big(r(x,a)+\gamma\cdot U(x')\big),
 \]
 \end{small}
On the other hand, utilizing the LCE model parameterization and the policy parameterization $\mu(a|x)=\pi\circ E(a|x)$, where $\pi\circ E$ corresponds to sampling a latent state $z$ from the LCE encoder $E$ and applying a latent space policy $\pi$, we also define an approximate Bellman operator for any observation $x\in\mathcal X$ one can re-write this function as
 \begin{small}
  \[
T^\prime_\mu[U](x):=\int_{z\in\mathcal Z} dE(z|x)\cdot \int_{a\in\mathcal A} d\pi(a|z)\cdot\int_{z'\in\mathcal Z}d\dyn(z'|z,a)\cdot\int_{x'\in\mathcal X}d\dec(x'|z')\cdot (r(x,a)+\gamma U(x')).
 \]
 \end{small}
 On the other hand for any value function $V:\mathcal Z\rightarrow\mathbb R$, consider the following latent-space Bellman operator:
  \begin{small}
   \[
\mathcal T_{\pi}[V](z):= \int_{a\in\mathcal A} d\pi(a|z)\cdot \int_{z'\in\mathcal Z}dF(z'|z,a)\cdot\left(\bar{r}(z,a)+\gamma\cdot V(z')\right),
 \]
  \end{small}
 where $\bar{r}(z,a)$ is a latent reward function that approximates the corresponding observation reward induced by the policy $\pi\circ E$, i.e., $|\int_{a\in\mathcal A}\int_{z\in\mathcal Z}d\pi(a|z)d\enc(z|x)(\bar{r}(z,a)-r(x,a))|\approx 0$, and $V:\mathcal Z\rightarrow\mathbb R$ is the latent value function. 
 
Using the results from \cite{Farahmand18IV}, first we bound the difference of the observation Bellman backup ${T}_\mu[U]$, for $\mu=\pi\circ\enc$, and the latent Bellman backup $T^\prime_\mu[U]$ w.r.t. an arbitrary value function $V$ using the following inequality:
\begin{small}
\begin{equation}\label{eq:prediction}
\begin{split}
\big|T_\mu[U](x) &- T^\prime_\mu[U](x) \big|\leq \gamma \|U\|_\infty\cdot D_{\text{TV}}\bigg(P_{\pi\circ\enc}(\cdot|x)||(\dec\circ\dyn\circ\pi\circ\enc)(\cdot|x)\bigg)\\
&\leq \frac{\gamma \|U\|_\infty}{\sqrt{2}}\cdot \sqrt{D_{\text{KL}}\bigg(P_{\pi\circ\enc}(\cdot|x)||(\dec\circ\dyn\circ\pi\circ\enc)(\cdot|x)\bigg)},\,\,\forall x\in\mathcal X,
\end{split}
\end{equation}
\end{small}
in which $\|U\|_\infty=\max_{x\in\mathcal X}|U(x)|$, and the second inequality is by Pinsker inequality. In a $\gamma$-discounting MDP, whose immediate reward is bounded by magnitude $R_{\max}$, this quantity is bounded by $R_{\max}/(1-\gamma)$. This implies that the difference of these Bellman operators can be bounded by a prediction loss.
Alternatively one can derive the above TV-divergence without the dependency of the policy by considering the worst-case actions as follows:
\begin{small}
\begin{equation}\label{eq:prediction_worst_action}
\begin{split}
D_{\text{TV}}\bigg(P_{\pi\circ\enc}(\cdot|x)||(\dec\circ\dyn\circ\pi\circ\enc)&(\cdot|x)\bigg) \leq \max_{a\in\mathcal A} \,D_{\text{TV}}\bigg(P(\cdot|x,a)||(\dec\circ\dyn\circ\enc)(\cdot|x,a)\bigg)\\
\leq & \frac{1}{\sqrt{2}}\max_{a\in\mathcal A} \,\sqrt{D_{\text{KL}}\bigg(P(\cdot|x,a)||(\dec\circ\dyn\circ\enc)(\cdot|x,a)\bigg)}.
\end{split}
\end{equation}
\end{small}
This upper bound corresponds to the prediction loss in PCC.

Second, for any arbitrary observation value function $U$, we have the following result that to connects the observation Bellman residual $\left|T^\prime_\mu[U](x)-U(x)\right|$ and the latent Bellman residual $\int_{z\in\mathcal Z}d\enc(z|x)\left|\mathcal T_{\pi}[V](z)-V(z)\right|$ when the latent value function is $V(z)=\int_{\widetilde{x}\in\mathcal X}d\dec(\widetilde{x}|z) U(\widetilde{x})$. 
\begin{lemma}\label{lem:1}
For any observation $x\in\mathcal X$, encoder-transition-decoder tuple $(\enc,\dyn,\dec)$, latent-space policy $\pi$, such that the policy is parameterized as $\mu=\pi\circ\enc$, and value function $V$, such that the latent function is defined as $V(z)=\int_{\widetilde{x}\in\mathcal X}d\dec(\widetilde{x}|z) U(\widetilde{x})$, the following statement holds:
\begin{small}
\begin{align}
&\left|T^\prime_\mu[U](x)\!-\!U(x)\right|\geq \left|\int_{z\in\mathcal Z}d\enc(z|x)(\mathcal T_{\pi}[V](z)-V(z))\right|\label{eq:lem1}\\
&\!-\!\left|\int_{z\in\mathcal Z}d\enc(z|x)\int_{a\in\mathcal A}d\pi(a|z)(r(x,a)-\bar{r}(z,a))\right|-\frac{R_{\max}}{1-\gamma}\sqrt{\frac{-1}{2}\int_{z\in\mathcal Z}d\enc(z|x)\log \dec(x|z)},\,\,\,\forall x\in\mathcal X\nonumber.
\end{align}
\end{small}
\end{lemma}
\begin{proof}
Using the definitions of Bellman operators, we have the following chain of inequalities:
\begin{small}
 \[
 \begin{split}
&\left|T^\prime_\mu[U](x)-\int_{z\in\mathcal Z}d\enc(z|x)\mathcal T_{\pi}[V](z)\right| \leq\left|\int_{z\in\mathcal Z}d\enc(z|x)\int_{a\in\mathcal A}d\pi(a|z)\cdot\right.\\
&\quad\left.\left(\int_{x'\in\mathcal X}\int_{z'\in\mathcal Z}d\dec(x'|z')\cdot d\dyn(z'|z,a)\cdot\gamma\cdot U(x')-\int_{z'\in\mathcal Z}d\dyn(z'|z,a)\cdot\gamma\cdot V(z')\right)\right|\\
&\quad+\left|\int_{z\in\mathcal Z}d\enc(z|x)\int_{a\in\mathcal A} d\pi(a|z)\int_{z'\in\mathcal Z} d\dyn(z'|z,a)\cdot\left(\int_{x'\in\mathcal X}d\dec(x'|z')r(x,a)-\bar{r}(z,a)\right)\right|\\
\leq &\left|\int_{z\in\mathcal Z}d\enc(z|x)\int_{a\in\mathcal A}d\pi(a|z)(r(x,a)-\bar{r}(z,a))\right|,
 \end{split}
 \]
 \end{small}
 where the inequalities are based on triangular inequality, and the first term in this inequality vanishes due to the definition $V(z)=\int_{\widetilde{x}\in\mathcal X}d\dec(\widetilde{x}|z) U(\widetilde{x})$.
Furthermore, using basic arithmetic manipulations and triangular inequality, at any $x\in\mathcal X$ we bound the Bellman residual $\left|T^\prime_\mu[U](x)-U(x)\right|$ as
\begin{small}
 \begin{align}
 &\bigg|T^\prime_\mu[U](x)-U(x)\bigg|\\
 =&\bigg|T^\prime_\mu[U](x)-\int_{z\in\mathcal Z}d\enc(z|x)(\mathcal T_{\pi}[V](z)-V(z))+\int_{z\in\mathcal Z}d\enc(z|x)(\mathcal T_{\pi}[V](z)-V(z))-U(x)\bigg|\nonumber\\
 \geq & \bigg|\int_{z\in\mathcal Z}d\enc(z|x)(\mathcal T_{\pi}[V](z)-V(z))\bigg|-\bigg|T^\prime_\mu[U](x)-\int_{z\in\mathcal Z}d\enc(z|x)\mathcal T_{\pi}[V](z)\bigg|-\left|U(x)-\int_{z\in\mathcal Z}d\enc(z|x)V(z)\right|\nonumber\\
 \geq&  \bigg|\int_{z\in\mathcal Z}d\enc(z|x)(\mathcal T_{\pi}[V](z)-V(z))\bigg| -\left|\int_{z\in\mathcal Z}d\enc(z|x)\int_{a\in\mathcal A}d\pi(a|z)(r(x,a)-\bar{r}(z,a))\right|\nonumber\\
 &-\left|\int_{\widetilde{x}\in\mathcal X}\left(\int_{z\in\mathcal Z}d\enc(z|x)d\dec(\widetilde{x}|z)-d\mathbf 1\{x=\widetilde{x}\}\right)U(\widetilde{x})\right|.\nonumber
 \end{align}
 \end{small}
 The last term in the above inequality can be further upper-bounded as follows:
 \begin{small}
 \[
\left|\int_{\widetilde{x}\in\mathcal X}\left(\int_{z\in\mathcal Z}d\enc(z|x)d\dec(\widetilde{x}|z)-d\mathbf 1\{x=\widetilde{x}\}\right)U(\widetilde{x})\right|\leq \int_{\widetilde{x}\in\mathcal X} \left|\int_{z\in\mathcal Z}d\enc(z|x)\dec(\widetilde{x}|z)-\mathbf 1\{x=\widetilde{x}\}\right|\cdot\|U\|_\infty.
 \]
 \end{small}
 Furthermore this upper bound can be further bounded by
 \begin{small}
\[
\begin{split}
&\int_{\widetilde{x}\in\mathcal X}\int_{z\in\mathcal Z}d\enc(z|x)\left|\dec(\widetilde{x}|z)-\mathbf 1\{x=\widetilde{x}\}\right|\leq\int_{z\in\mathcal Z}d\enc(z|x)\sqrt{\frac{1}{2}D_{\text{KL}}(\mathbf 1\{\cdot=x\}||\dec(\cdot|z))}\\
=&\int_{z\in\mathcal Z}d\enc(z|x)\sqrt{\frac{-1}{2}\log \dec(x|z)}\leq \sqrt{\frac{-1}{2}\int_{z\in\mathcal Z}d\enc(z|x)\log \dec(x|z)},
\end{split}
\]
\end{small}
where the first inequality follows from Pinsker's inequality and the second inequality follows from the concavity of $\sqrt{(\cdot)}$. Combining this result with the above inequality completes the proof.
 \end{proof}
 
 This right side of inequality (\ref{eq:lem1}) contains several terms. The first corresponds to the latent Bellman residual error, the second corresponds to the latent reward estimation error w.r.t. policy $\pi$, and the third term is a reconstruction loss in the encoder-decoder path, which is commonly found in training auto-encoders (and is also a regularization term in PCC). Utilizing the inequality in \eqref{eq:prediction} and this lemma, one can further show that for any $V:\mathcal X\rightarrow\mathbb R$ and at any $x\in\mathcal X$,
 \begin{small}
\begin{align}
&\left|T_\mu[U](x)\!-\!U(x)\right|=\left|T_\mu[U](x)\!-\!T^\prime_\mu[U](x)+T^\prime_\mu[U](x)\!-\!U(x)\right|\geq \left|T^\prime_\mu[U](x)\!-\!U(x)\right|-\left|T_\mu[U](x)\!-\!T^\prime_\mu[U](x)\right|\nonumber\\
\geq& \left|\int_{z\in\mathcal Z}d\enc(z|x)(\mathcal T_{\pi}[V](z)-V(z))\right|-\frac{\gamma R_{\max}}{\sqrt{2}(1-\gamma)}\cdot \sqrt{D_{\text{KL}}\bigg(P_{\pi\circ\enc}(\cdot|x)||(\dec\circ\dyn\circ\pi\circ\enc)(\cdot|x)\bigg)}\nonumber\\
&\!-\!\left|\int_{z\in\mathcal Z}d\enc(z|x)\int_{a\in\mathcal A}d\pi(a|z)(r(x,a)-\bar{r}(z,a))\right|-\frac{R_{\max}}{1-\gamma}\sqrt{\frac{-1}{2}\int_{z\in\mathcal Z}d\enc(z|x)\log \dec(x|z)}\label{eq:lem_1_plus}.
\end{align}
\end{small}

Fourth, for any arbitrary latent value function $V$, we have the following result that to connects the observation Bellman residual $\left|T_\mu[U](x)-U(x)\right|$ and the latent Bellman residual $\int_{z\in\mathcal Z}d\enc(z|x)\left|\mathcal T_{\pi}[V](z)-V(z)\right|$ when the observation value function is $U(x)=\int_{\widetilde{z}\in\mathcal Z}d\enc(\widetilde{z}|x) V(\widetilde{z})$. 
 \begin{lemma}\label{lem:2}
For any observation $x\in\mathcal X$, encoder-transition-decoder tuple $(\enc,\dyn,\dec)$, latent-space policy $\pi$, such that the policy is parameterized as $\mu=\pi\circ\enc$, and latent value function $V$, such that the function is defined as $ U(x)=\int_{\widetilde{z}\in\mathcal Z}d\enc(\widetilde{z}|x) V(\widetilde{z})$, the following statement holds:
\begin{small}
\begin{align}
&\left|\int_{z\in\mathcal Z}d\enc(z|x)\cdot\big(\mathcal T_{\pi}[V](z)-V(z)\big)\right|\geq \big|T_\mu[U](x) - U(x) \big| -\left|\int_{z\in\mathcal Z}d\enc(z|x)\int_{a\in\mathcal A}d\pi(a|z)(r(x,a)-\bar{r}(z,a))\right|\nonumber\\
&\,\,- \frac{\gamma R_{\max}}{\sqrt{2}(1-\gamma)}\sqrt{  \int_{x'\in\mathcal X}dP_{\pi\circ\enc}(x'|x)\cdot D_\text{KL}\left( \enc(\cdot|x')||(\dyn_\pi\circ\enc)(\cdot|x)\right)}.
\end{align}
\end{small}
 \end{lemma}
\begin{proof}
 Using the definition $ U(x)=\int_{\widetilde{z}\in\mathcal Z}d\enc(\widetilde{z}|x) V(\widetilde{z})$ and the triangular inequality, 
 \begin{small}
 \[
 \left|\int_{z\in\mathcal Z}d\enc(z|x)\cdot\big(\mathcal T_{\pi}[V](z)-V(z)\big)\right|\geq \big|T_\mu[U](x) - U(x) \big| -  \left|\int_{z\in\mathcal Z}d\enc(z|x)\cdot\mathcal T_{\pi}[V](z)-T_\mu[U](x) \right| 
 \]
 \end{small}
 the proof of this lemma is completed by bounding the difference of the first term via the following inequality:
 \begin{small}
 \[
 \begin{split}
&\left|\int_{z\in\mathcal Z}d\enc(z|x)\cdot\mathcal T_{\pi}[V](z)-T_\mu[U](x) \right|\leq \left|\int_{z\in\mathcal Z}d\enc(z|x)\int_{a\in\mathcal A}d\pi(a|z)(r(x,a)-\bar{r}(z,a))\right|+\\
&\quad\gamma \left|\int_{a\in\mathcal A} d\mu(a|x)\int_{x'\in\mathcal X}dP(x'|x,a)\cdot \int_{z'\in\mathcal Z}d\enc(z'|x')V(z')-\int_{z\in\mathcal Z} dE(z|x)\cdot \int_{a\in\mathcal A} d\pi(a|z)\cdot\int_{z'\in\mathcal Z}d\dyn(z'|z,a)\cdot V(z')\right|
 \end{split}
 \]
 \end{small}
 in which the second term can be further upper bounded as follows:
 \begin{small}
 \[
 \begin{split}
 &\left|\int_{a\in\mathcal A} d\mu(a|x)\int_{x'\in\mathcal X}dP(x'|x,a)\cdot \int_{z'\in\mathcal Z}d\enc(z'|x')V(z')-\int_{z\in\mathcal Z} dE(z|x)\cdot \int_{a\in\mathcal A} d\pi(a|z)\cdot\int_{z'\in\mathcal Z}d\dyn(z'|z,a)\cdot V(z')\right|\\
\leq&\int_{z'\in\mathcal Z}\left|\int_{a\in\mathcal A} d\mu(a|x)\int_{x'\in\mathcal X}dP(x'|x,a)\cdot \enc(z'|x')-\int_{z\in\mathcal Z} dE(z|x)\cdot \int_{a\in\mathcal A} d\pi(a|z)\cdot\dyn(z'|z,a)\right| \cdot \|V\|_\infty\\
=& D_\text{TV}\left(\int_{a\in\mathcal A} d\mu(a|x)\int_{x'\in\mathcal X}dP(x'|x,a)\cdot \enc(\cdot|x')||\int_{z\in\mathcal Z} dE(z|x)\cdot \int_{a\in\mathcal A} d\pi(a|z)\cdot\dyn(\cdot|z,a)\right)\cdot \|V\|_\infty\\
\leq& \frac{\|V\|_\infty}{\sqrt{2}}\sqrt{  \int_{x'\in\mathcal X}dP_{\pi\circ\enc}(x'|x)\cdot D_\text{KL}\left( \enc(\cdot|x')||(\dyn_\pi\circ\enc)(\cdot|x)\right)}.
 \end{split}
 \]
 \end{small}
 where the first equality follows definition of TV-divergence and the second inequality follows directly from Pinsker inequality and joint convexity of $D_{\text{KL}}(x||y)$. The proof of this lemma can be completed by combining the above results and using $\|V\|_\infty\leq R_{\max}/(1-\gamma)$.
\end{proof}

Notice that since both the observation Bellman operator $T_\mu$ and the latent Bellman operator $\mathcal T_{\pi}$ are contraction mappings, there exists a unique solution $U_\mu:\mathcal{X}\rightarrow\mathbb{R}$ to the observation fixed point equation $T_\mu[U](x)=U(x)$, $\forall x\in\mathcal X$, a unique solution $V_{\pi}:\mathcal Z\rightarrow\mathbb R$ to the latent fixed point equation $\mathcal T_{\pi}[V](z)=V(z)$, $\forall z\in\mathcal Z$.
Together with the result in (\ref{eq:lem_1_plus}), we can immediately show that theorem \ref{thm:policy_eval}, which connects the optimal observation bellman residual error and the optimal latent bellman residual error.

\begin{theorem}\label{thm:policy_eval}
Let $ V_\pi\circ \enc(x)=\int_{\widetilde{z}\in\mathcal Z}d\enc(\widetilde{z}|x) V_{\pi}(\widetilde{z})$, be the observation value function in which $V_{\pi}$ is the solution of the soft latent fixed-point equation $\mathcal T_{\pi}[V](z)=V(z)$ w.r.t. latent policy $\pi$, encoder-transition-decoder tuple $(\enc,\dyn,\dec)$, and parameterized observation-based policy $\mu=\pi\circ\enc$. Then the following statement holds at any $x\in\mathcal X$:
\begin{small}
\[
\begin{split}
&\big|T_\mu[V_\pi\circ \enc](x) - V_\pi\circ \enc(x) \big|\leq \frac{R_{\max}}{1-\gamma}\sqrt{\frac{-1}{2}\int_{z\in\mathcal Z}d\enc(z|x)\log \dec(x|z)}\\
&+2\left|\int_{z\in\mathcal Z}d\enc(z|x)\int_{a\in\mathcal A}d\pi(a|z)(r(x,a)-\bar{r}(z,a))\right|+ \frac{\gamma R_{\max}}{\sqrt{2}(1-\gamma)}\left\{ \sqrt{D_{\text{KL}}\bigg(P_{\pi\circ\enc}(\cdot|x)||(\dec\circ\dyn\circ\pi\circ\enc)(\cdot|x)\bigg)}\right.\\
&+\left.\sqrt{  \int_{x'\in\mathcal X}dP_{\pi\circ\enc}(x'|x)\cdot D_\text{KL}\left( \enc(\cdot|x')||(\dyn_\pi\circ\enc)(\cdot|x)\right)}\right\}.
\end{split}
\]
\end{small}
\end{theorem}
\begin{proof}
Using Lemma \ref{lem:1} with $V=U_\mu$, the fixed-point solution $T_\mu[U_\mu](x)=U_\mu(x)$, and denoting by $\hat{V}_\mu(z)=\int_{\widetilde{x}\in\mathcal X}d\dec(\widetilde{x}|z) U_\mu(\widetilde{x})$, for any $x\in\mathcal X$, we have
\begin{small}
\begin{align}
&\left|T_\mu[U_\mu](x)\!-\!U_\mu(x)\right|\nonumber\\
\geq& \left|\int_{z\in\mathcal Z}d\enc(z|x)(\mathcal T_{\pi}[\hat{V}_\mu](z)-\hat{V}_\mu(z))\right|-\frac{\gamma R_{\max}}{\sqrt{2}(1-\gamma)}\cdot \sqrt{D_{\text{KL}}\bigg(P_{\pi\circ\enc}(\cdot|x)||(\dec\circ\dyn\circ\pi\circ\enc)(\cdot|x)\bigg)}\nonumber\\
&\!-\!\left|\int_{z\in\mathcal Z}d\enc(z|x)\int_{a\in\mathcal A}d\pi(a|z)(r(x,a)\!-\!\bar{r}(z,a))\right|-\frac{R_{\max}}{1-\gamma}\sqrt{\frac{-1}{2}\int_{z\in\mathcal Z}d\enc(z|x)\log \dec(x|z)}.\nonumber
\end{align}
\end{small}

Using the fact that $V_\pi$ is the fixed-point solution, i.e., $\mathcal T_\pi[V_\pi](z)=V_\pi(z)$, we have
\begin{small}
\[
\left|\int_{z\in\mathcal Z}d\enc(z|x)(\mathcal T_{\pi}[\hat{V}_\mu](z)-\hat{V}_\mu(z))\right|\geq \left|\int_{z\in\mathcal Z}d\enc(z|x)(\mathcal T_{\pi}[V_\pi](z)-V_\pi(z))\right|=0.
\]
\end{small}

On the other hand, using Lemma \ref{lem:2} with $V=U_\mu$, and with the definition of $V_\pi\circ \enc(x)=\int_{\widetilde{z}\in\mathcal Z}d\enc(\widetilde{z}|x) V_{\pi}(\widetilde{z})$ we can further show that
\begin{small}
\begin{align}
&\left|\int_{z\in\mathcal Z}d\enc(z|x)(\mathcal T_{\pi}[\hat{V}_\mu](z)-\hat{V}_\mu(z))\right|\geq \big|T_\mu[V_\pi\circ \enc](x) - V_\pi\circ \enc(x) \big| -\left|\int_{z\in\mathcal Z}d\enc(z|x)\int_{a\in\mathcal A}d\pi(a|z)(r(x,a)-\bar{r}(z,a))\right|\nonumber\\
&\,\,- \frac{\gamma R_{\max}}{\sqrt{2}(1-\gamma)}\sqrt{  \int_{x'\in\mathcal X}dP_{\pi\circ\enc}(x'|x)\cdot D_\text{KL}\left( \enc(\cdot|x')||(\dyn_\pi\circ\enc)(\cdot|x)\right)}.\nonumber
\end{align}
\end{small}

Together these inequalities imply that 
\begin{small}
\begin{align}
&\left|T_\mu[U_\mu](x)\!-\!U_\mu(x)\right|\geq \big|T_\mu[V_\pi\circ \enc](x) - V_\pi\circ \enc(x) \big| -2\left|\int_{z\in\mathcal Z}d\enc(z|x)\int_{a\in\mathcal A}d\pi(a|z)(r(x,a)-\bar{r}(z,a))\right|\nonumber\\
&\,\,- \frac{\gamma R_{\max}}{\sqrt{2}(1-\gamma)}\sqrt{  \int_{x'\in\mathcal X}dP_{\pi\circ\enc}(x'|x)\cdot D_\text{KL}\left( \enc(\cdot|x')||(\dyn_\pi\circ\enc)(\cdot|x)\right)}.\nonumber\\
&\,\,-\frac{\gamma R_{\max}}{\sqrt{2}(1-\gamma)}\cdot \sqrt{D_{\text{KL}}\bigg(P_{\pi\circ\enc}(\cdot|x)||(\dec\circ\dyn\circ\pi\circ\enc)(\cdot|x)\bigg)}-\frac{R_{\max}}{1-\gamma}\sqrt{\frac{-1}{2}\int_{z\in\mathcal Z}d\enc(z|x)\log \dec(x|z)}.\nonumber
\end{align}
\end{small}
By recalling $U_\mu$ is the fixed-point solution, i.e., $T_\mu[U_\mu](x)=U_\mu(x)$, $\forall x\in\mathcal X$, the proof is completed by setting the left side to be zero.
\end{proof}
This theory shows that the observation Bellmen residual error w.r.t. value function $V_\pi\circ \enc$, where $V_{\pi}$ is the optimal latent value function (w.r.t. soft Bellman fixed-point equation), depends on (i) the prediction error, (ii) the consistency term, (iii) latent reward estimation error, and (iv) the encoder-decoder reconstruction error.
Using analogous derivations as in (\ref{eq:prediction_worst_action}) for the prediction term, we can further derive a observation Bellman residual error upper bound w.r.t. value function $V_\pi\circ \enc$ that does not depend on the policy. 
\begin{corollary}\label{coro:policy_eval}
Let $ V_\pi\circ \enc(x)=\int_{\widetilde{z}\in\mathcal Z}d\enc(\widetilde{z}|x) V_{\pi}(\widetilde{z})$, be the observation value function in which $V_{\pi}$ is the solution of the soft latent fixed-point equation $\mathcal T_{\pi}[V](z)=V(z)$ w.r.t. latent policy $\pi$, encoder-transition-decoder tuple $(\enc,\dyn,\dec)$, and parameterized observation-based policy $\mu=\pi\circ\enc$. Then the following statement holds at any $x\in\mathcal X$:
\begin{small}
\[
\begin{split}
&\big|T_\mu[V_\pi\circ \enc](x) - V_\pi\circ \enc(x) \big|\leq \frac{R_{\max}}{1-\gamma}\sqrt{\frac{-1}{2}\int_{z\in\mathcal Z}d\enc(z|x)\log \dec(x|z)}\\
&\qquad+2\cdot\max_{a\in\mathcal A}\left\{\left|r(x,a)-\int_{z}d\enc(z|x)\bar{r}(z,a)\right|+ \frac{\gamma R_{\max}}{\sqrt{2}(1-\gamma)}\left\{ \sqrt{ D_{\text{KL}}\bigg(P(\cdot|x,a)||(\dec\circ\dyn\circ\enc)(\cdot|x,a)\bigg)}\right.\right.\\
&\qquad\qquad\qquad\qquad+\left.\left.\sqrt{ \int_{x'\in\mathcal X}dP(x'|x,a)\cdot D_\text{KL}\left(\enc(\cdot|x')||\dyn\circ E(\cdot|x,a)\right)}\right\}\right\}.
\end{split}
\]
\end{small}
\end{corollary}
  
  \begin{lemma}\label{lem:policy_eval}
 The observation value function $U_{\pi\circ\enc}$ w.r.t. policy $\pi\circ\enc$ satisfies the following bound
 \begin{small}
 \[
 \left|V_{\pi}\circ\enc(x)- U_{\pi\circ\enc}(x)\right|\leq \frac{\gamma}{1-\gamma}\Delta(\enc,\dyn,\dec,\bar{r},\pi,x),\,\, \forall x\in\mathcal X.
 \]
 \end{small}
 where the error term is given by
 \begin{small}
\[
\begin{split}
\Delta(\,&\enc,\dyn,\dec,\bar{r},\pi,x)=\frac{R_{\max}}{1-\gamma}\sqrt{\frac{-1}{2}\int_{z}d\enc(z|x)\log \dec(x|z)} \\ 
&+ 2\big|\int_{z}d\enc(z|x)\int_a d\pi(a|z)(r(x,a)-\bar{r}(z,a))\big| + \frac{\gamma R_{\max}}{\sqrt{2}(1-\gamma)}\sqrt{D_{\text{KL}}\big(P_{\pi\circ E}(\cdot|x) \; || \; (D\circ F_\pi\circ E)(\cdot|x)\big)} \\
&+ \frac{\gamma R_{\max}}{\sqrt{2}(1-\gamma)}\sqrt{ \int_{x'\in\mathcal X}dP_{\pi\circ\enc}(x'|x)\cdot D_\text{KL}\left( \enc(\cdot|x')||(\dyn_\pi\circ\enc)(\cdot|x)\right)}.
\end{split}
\]
\end{small}
\end{lemma}
 \begin{proof}
 To prove the right side of the approximate policy evaluation inequality, recall from Theorem \ref{thm:policy_eval} with $\mu=\pi\circ\enc$ and LCE-reward models $(\enc,\dyn,\dec,\bar{r})$ that 
 \begin{small}
 \[
 T_{\pi\circ\enc}[V_{\pi}\circ\enc](x)\leq V_{\pi}\circ\enc(x)+ \Delta(\enc,\dyn,\dec,\bar{r},\pi,x),\,\,\forall x\in\mathcal X.
 \]
 \end{small}
 Applying the Bellman operator $T_{\pi\circ\enc}$ on both sides we get
 \begin{small}
 \[
 T^2_{\pi\circ\enc}[V_{\pi}\circ\enc](x)\leq  T_{\pi\circ\enc}[V_{\pi}\circ\enc](x)+ \gamma\Delta(\enc,\dyn,\dec,\bar{r},\pi,x),\,\,\forall x\in\mathcal X.
 \]
 \end{small}
 Proceeding similarly it follows that
 \begin{small}
 \[
T^\ell_{\pi\circ\enc}[V_{\pi}\circ\enc](x)\leq  T^{\ell-1}_{\pi\circ\enc}[V_{\pi}\circ\enc](x)+ \gamma\Delta(\enc,\dyn,\dec,\bar{r},\pi,x),\,\,\forall x\in\mathcal X.
 \]
 \end{small}
 Therefore, for every $k>0$
 \begin{small}
 \[
 \begin{split}
 T^k_{\pi\circ\enc}[V_{\pi}\circ\enc](x)-V_{\pi}\circ\enc(x)&\leq  \sum_{\ell=1}^kT^\ell_{\pi\circ\enc}[V_{\pi}\circ\enc](x)-T^{\ell-1}_{\pi\circ\enc}[V_{\pi}\circ\enc](x)\\
 &\leq \frac{\gamma}{1-\gamma}\cdot\Delta(\enc,\dyn,\dec,\bar{r},\pi,x),\,\,\forall x\in\mathcal X.
 \end{split}
 \]
 \end{small}
 Taking $k\rightarrow\infty$ we then obtain $U_{\pi\circ\enc}(x)\leq V_{\pi}\circ\enc(x)+\frac{\gamma}{1-\gamma}\Delta(\enc,\dyn,\dec,\bar{r},\pi,x)$. On the other hand, the left side of the policy evaluation inequality follows analogous arguments. This completes the proof of the approximate policy evaluation. 
 \end{proof}
 
\section{Approximate Policy Iteration with CARL}
\label{sec:API-CARL-Proofs}

Consider the following latent policy iteration procedure that optimizes the policy in form of $\mu\circ\enc$. Starting at iteration $i=0$, given an initial observation policy $\mu^{(0)}$, an initial observation value function $U^{(0)}$, a LCE model $(\enc^{(0)},\dyn^{(0)},\dec^{(0)})$, and a latent reward model $\bar{r}^{(0)}$, do
\begin{enumerate}
\item Compute the distilled latent policy by $\pi^{(i)}\leftarrow\arg\min_{\mu}D_{\text{KL}}(\pi\circ\enc^{(i)}(\cdot|x)||\mu^{(i)}(\cdot|x))$
    \item Compute the $\pi^{(i)}$-induced latent value function $ V_{\pi^{(i)}}(z)=\lim_{n\rightarrow\infty}T^n_{\pi^{(i)}}[ V^{(i)}](z)$, $\forall z$ w.r.t. models $(\dyn^{(i)},\bar{r}^{(i)})$ and the state-action latent value function $ Q_{\pi^{(i)}}(z,a):= \bar{r}^{(i)}(z,a)+\gamma\int_{z'\in\mathcal X}dF^{(i)}(z'|z,a) V_{\pi^{(i)}}(z')$
    \item Compute the updated latent policy $\pi^{(i+1)}(\cdot|z)\in\arg\max_{p\in\Delta}\int_{a\in\mathcal A}dp(a)\cdot Q_{\pi^{(i)}}(z,a)$, and the updated observation policy $\mu^{(i+1)}(\cdot|x)=\pi^{(i+1)}\circ\enc^{(i)}(\cdot|x)$
     \item Update the LCE model $(\enc^{(i+1)},\dyn^{(i+1)},\dec^{(i+1)})$ and the latent reward model $\bar{r}^{(i+1)}$ by minimizing the loss $\Delta(\enc,\dyn,\dec,\bar{r},\pi^{(i+1)})$
\end{enumerate}
Repeat step 1-4 with the updated value function $U^{(i+1)}=U_{\pi^{(i)}}$.

Equipped with this technical property and the policy evaluation result in Theorem \ref{thm:policy_eval}, we can now provide a policy improvement result on the above proposed procedure.

\begin{theorem}\label{thm:policy_improve}
For any observation $x\in\mathcal X$, the latent policy iteration procedure has the approximate policy improvement property:
\begin{small}
\begin{equation}\label{eq:policy_improve_error}
U_{\mu^{(i+1)}}(x) \geq U_{\mu^{(i)}}(x)-\frac{\gamma}{1-\gamma}\sum_{\pi\in\{\pi^{(i)},\pi^{(i+1)}\}}\Delta(\enc^{(i)},\dyn^{(i)},\dec^{(i)},\bar{r}^{(i)},\pi,x)- \frac{R_{\max}}{1-\gamma}\cdot D_{\text{KL}}(\mu^{(i)}(\cdot|x)||\pi^{(i)}\circ\enc^{(i)}(\cdot|x)).
\end{equation}
\end{small}
where the first error term is given by
\begin{small}
\[
\begin{split}
\Delta(\,&\enc,\dyn,\dec,\bar{r},\pi,x)=\frac{R_{\max}}{1-\gamma}\sqrt{\frac{-1}{2}\int_{z}d\enc(z|x)\log \dec(x|z)} \\ 
&+ 2\big|\int_{z}d\enc(z|x)\int_a d\pi(a|z)(r(x,a)-\bar{r}(z,a))\big| + \frac{\gamma R_{\max}}{\sqrt{2}(1-\gamma)}\sqrt{D_{\text{KL}}\big(P_{\pi\circ E}(\cdot|x) \; || \; (D\circ F_\pi\circ E)(\cdot|x)\big)} \\
&+ \frac{\gamma R_{\max}}{\sqrt{2}(1-\gamma)}\sqrt{ \int_{x'\in\mathcal X}dP_{\pi\circ\enc}(x'|x)\cdot D_\text{KL}\left( \enc(\cdot|x')||(\dyn_\pi\circ\enc)(\cdot|x)\right)}.
\end{split}
\]
\end{small}
\end{theorem}
 \begin{proof}
  Using the policy evaluation property, we have
  \begin{small}
 \[
U_{\pi^{(i)}\circ\enc^{(i)}}(x)\leq V_{\pi^{(i)}}\circ\enc^{(i)}(x)+\frac{\gamma}{1-\gamma}\Delta(\enc^{(i)},\dyn^{(i)},\dec^{(i)},\bar{r}^{(i)},\pi^{(i)},x), \,\,\,\forall x\in\mathcal X.
 \] 
 \end{small}
  Applying Bellman operator $T_{\pi^{(i)}\circ\enc^{(i)}}$ on both sides and noticing that $T_{\pi^{(i)}\circ\enc^{(i)}}[U](x)\leq T[U](x)$ uniformly at $x\in\mathcal X$ for any observation value function $U$, we can then show that for any $x\in\mathcal X$,
  \begin{small}
  \begin{equation}\label{eq:intermediate}
  \begin{split}
&U_{\pi^{(i)}\circ\enc^{(i)}}(x)\\
\leq& {V}_{\pi^{(i)}}\circ\enc^{(i)}(x)+\frac{\gamma}{1-\gamma}\Delta(\enc^{(i)},\dyn^{(i)},\dec^{(i)},\bar{r}^{(i)},\pi^{(i)},x)\\
=& \int_{z\in\mathcal Z}d\enc^{(i)}(z|x)\mathcal T_{\pi^{(i)}}[V_{\pi^{(i)}}](z)+\frac{\gamma}{1-\gamma}\cdot\Delta(\enc^{(i)},\dyn^{(i)},\dec^{(i)},\bar{r}^{(i)},\pi^{(i)},x)\\
\leq& \int_{z\in\mathcal Z}d\enc^{(i)}(z|x)\mathcal T[V_{\pi^{(i)}}](z)+\frac{\gamma}{1-\gamma}\cdot\Delta(\enc^{(i)},\dyn^{(i)},\dec^{(i)},\bar{r}^{(i)},\pi^{(i)},x)\\
=& \int_{z\in\mathcal Z}d\enc^{(i)}(z|x)\mathcal T_{\pi^{(i+1)}}[V_{\pi^{(i)}}](z)+\frac{\gamma}{1-\gamma}\cdot\Delta(\enc^{(i)},\dyn^{(i)},\dec^{(i)},\bar{r}^{(i)},\pi^{(i)},x)\\
\leq& \int_{z\in\mathcal Z}d\enc^{(i)}(z|x)\mathcal T_{\pi^{(i+1)}}[V_{\pi^{(i+1)}}](z)+\frac{\gamma}{1-\gamma}\cdot\Delta(\enc^{(i)},\dyn^{(i)},\dec^{(i)},\bar{r}^{(i)},\pi^{(i)},x)\\
=& \int_{z\in\mathcal Z}d\enc^{(i)}(z|x)V_{\pi^{(i+1)}}(z)+\frac{\gamma}{1-\gamma}\cdot\Delta(\enc^{(i)},\dyn^{(i)},\dec^{(i)},\bar{r}^{(i)},\pi^{(i)},x)\\
\leq&  \underbrace{U_{\pi^{(i+1)}\circ\enc^{(i)}}(x)}_{U_{\mu^{(i+1)}}(x)}+\frac{\gamma}{1-\gamma}\left\{\Delta(\enc^{(i)},\dyn^{(i)},\dec^{(i)},\bar{r}^{(i)},\pi^{(i)},x)+\Delta(\enc^{(i)},\dyn^{(i)},\dec^{(i)},\bar{r}^{(i)},\pi^{(i+1)},x)\right\}.
\end{split}
 \end{equation}
 \end{small}
The first inequality is based on Lemma \ref{lem:policy_eval}. The first equality is based on the property that $V_{\pi^{(i)}}$ is a unique solution to fixed-point equation $\mathcal T_{\pi^{(i)}}[V](z)=V(z)$. The second inequality is based on the definition 
\begin{small}
\[
\begin{split}
\mathcal T[V](z)&=\max_{p\in\Delta} \int_a dp(a)\left\{\bar{r}^{(i)}(z,a)+\gamma\int_{z'\in\mathcal X}dF^{(i)}(z'|z,a) V_{\pi^{(i)}}(z')\right\}\\
&\geq \int_a d\pi^{(i)}(a|z)\left\{\bar{r}^{(i)}(z,a)+\gamma\int_{z'\in\mathcal X}dF^{(i)}(z'|z,a) V_{\pi^{(i)}}(z')\right\}.
\end{split}
\]
\end{small}
The second equality is based on the definition of $\pi^{(i+1)}$.
The third inequality is based on the policy improvement property in latent policy iteration, i.e.,
\begin{small}
\[
\mathcal T[V_{\pi^{(i)}}]\geq \mathcal T_{\pi^{(i)}}[V_{\pi^{(i)}}]=V_{\pi^{(i)}}\implies V_{\pi^{(i+1)}}=\lim_{k\rightarrow\infty}\mathcal T_{\pi^{(i+1)}}^k[V_{\pi^{(i)}}]=\lim_{k\rightarrow\infty}\mathcal T^k[V_{\pi^{(i)}}]\geq V_{\pi^{(i)}},
\]
\end{small}
and the monotonicity property of latent Bellman operator.
The third equality is based on the fact that $V_{\pi^{(i+1)}}$ is a unique solution to fixed-point equation $\mathcal T_{\pi^{(i+1)}}[V](z)=V(z)$. The fourth inequality is again based on Lemma \ref{lem:policy_eval} (when $\pi=\pi^{(i+1)}$).

 Furthermore, considering the error from the distillation step,
 using the result from \cite{schulman2015trust}, one can show that 
 \begin{small}
 \begin{equation}
  U_{\mu^{(i)}}(x)\leq U_{\pi^{(i)}\circ\enc^{(i)}}(x)+ \frac{R_{\max}}{\sqrt{2}(1-\gamma)}\cdot \sqrt{D_{\text{KL}}(\pi^{(i)}\circ\enc^{(i)}(\cdot|x)||\mu^{(i)}(\cdot|x))}
 \end{equation}
 \end{small}
 Together this implies the result of the approximate policy improvement property.
 \end{proof}

\section{Offline CARL}
\label{app:offline_CARL}
  
Now consider the following offline latent policy iteration procedure that optimizes the policy in form of $\mu\circ\enc$. Starting at iteration $i=0$, given an initial observation policy $\mu^{(0)}$, an initial observation value function $U^{(0)}$, an offline LCE model $(\enc,\dyn,\dec)$, and an offline latent reward model $\bar{r}$, do
\begin{enumerate}
\item Compute the distilled latent policy by $\pi^{(i)}\leftarrow\arg\min_{\mu}D_{\text{KL}}(\pi\circ\enc(\cdot|x)||\mu^{(i)}(\cdot|x))$
    \item Compute the updated latent policy $\pi^{(i+1)}(\cdot|z)\in\arg\max_{p\in\Delta}\int_{a\in\mathcal A}dp(a)\cdot Q_{\pi^{(i)}}(z,a)$, and the updated observation policy $\mu^{(i+1)}(\cdot|x)=\pi^{(i+1)}\circ\enc(\cdot|x)$
\end{enumerate}
Repeat step 1-2 with the updated value function $U^{(i+1)}=U_{\pi^{(i)}}$ until convergence.
\begin{corollary}
For any observation $x\in\mathcal X$, the offline latent policy iteration procedure has the approximate policy improvement property
\begin{small}
\begin{equation}\label{eq:policy_improve_offline}
U_{\mu^{(i+1)}}(x) \geq U_{\mu^{(i)}}(x)-\frac{2\gamma}{1-\gamma}\max_{a\in\mathcal A}\Delta(\enc,\dyn,\dec,\bar{r}^{(i)},a,x),
\end{equation}
\end{small}
and the sub-optimality performance bound:
\begin{small}
\[
\lim_{i\rightarrow\infty}\|U_{\mu^{(i)}}(x) - U^*(x)\| \leq \frac{2\gamma}{(1-\gamma)^2}\max_{a\in\mathcal A,x\in\mathcal X}\Delta(\enc,\dyn,\dec,\bar{r},a,x),
\]
\end{small}
where the optimal value function is given by
\begin{small}
\[
U^*(x)=\mathbb E[\sum_{t=0}^\infty\gamma^tr_\mu(x_t) \mid P_\mu,\mu=\pi^*\circ\enc, x_0=x],\,\,\pi^*\in\arg\max_{\pi} \mathbb E[\sum_{t=0}^\infty\gamma^t\bar{r}_\pi(z_t) \mid F_\pi,z_0=z]
\]
\end{small}
and the action-dependent (and policy-independent) error term is given by
\begin{small}
\[
\begin{split}
\Delta(\,&\enc,\dyn,\dec,\bar{r},a,x)=\frac{R_{\max}}{1-\gamma}\sqrt{\frac{-1}{2}\int_{z}d\enc(z|x)\log \dec(x|z)} \\ 
&+ 2\big|r(x,a)-\int_{z}d\enc(z|x)\bar{r}(z,a)\big| + \frac{\gamma R_{\max}}{\sqrt{2}(1-\gamma)}\sqrt{D_{\text{KL}}\big(P(\cdot|x,a) \; || \; (D\circ F\circ E)(\cdot|x,a)\big)} \\
&+ \frac{\gamma R_{\max}}{\sqrt{2}(1-\gamma)}\sqrt{\int_{z}d\enc(z|x)\int_{x'}dP(x'|x,a)\cdot D_\text{KL}\big(E(\cdot|x') \; || \; F (\cdot|z,a)\big)}.
\end{split}
\]
\end{small}
\end{corollary}
\begin{proof}
Using the result from Theorem \ref{thm:policy_improve}, when the LCE model $(\enc,\dyn,\dec)$ and the latent reward model $\bar r$ do not change online, and there is no need for the distillation step. We also follow analogous arguments as in Corollary \ref{coro:policy_eval} that for $\Delta(\enc,\dyn,\dec,\bar{r}^{(i)},\pi,x)$ with the more conservative, worst-case error term over actions, i.e., $\max_{a\in\mathcal A}\Delta(\enc,\dyn,\dec,\bar{r}^{(i)},a,x)$, one can eliminate the dependencies on policies. Therefore, the policy improvement property can be simplified as in (\ref{eq:policy_improve_offline}). 

Denote by $
T_{\mu^*}[U](x)
$ the observation Bellman operator w.r.t. optimal latent policy $\pi^*$. Recall that $\int_{z\in\mathcal Z}d\enc(z|x)\mathcal T_{\pi^{(i+1)}}[V_{\pi^{(i+1)}}](z)=V_{\pi^{(i+1)}}\circ\enc(x)$. Using the results in Lemma \ref{lem:policy_eval}, we have the following chain on inequalities
\begin{small}
\[
\begin{split}
U_{\mu^{(i+1)}}(x)\geq&\int_{z\in\mathcal Z}d\enc(z|x)\mathcal T_{\pi^{(i+1)}}[V_{\pi^{(i+1)}}](z)-\frac{\gamma}{1-\gamma}\cdot\max_{a\in\mathcal A}\Delta(\enc,\dyn,\dec,\bar{r},a,x)\\
\geq& \int_{z\in\mathcal Z}d\enc(z|x)\mathcal T_{\pi^{(i+1)}}[V_{\pi^{(i)}}](z)-\frac{\gamma}{1-\gamma}\cdot\max_{a\in\mathcal A}\Delta(\enc,\dyn,\dec,\bar{r},a,x)\\
=& \int_{z\in\mathcal Z}d\enc(z|x)\mathcal T[V_{\pi^{(i)}}](z)-\frac{\gamma}{1-\gamma}\cdot\max_{a\in\mathcal A}\Delta(\enc,\dyn,\dec,\bar{r},a,x)\\
\geq& \int_{z\in\mathcal Z}d\enc(z|x)\mathcal T_{\pi^*}[V_{\pi^{(i)}}](x)-\frac{\gamma}{1-\gamma}\cdot\max_{a\in\mathcal A}\Delta(\enc,\dyn,\dec,\bar{r},a,x)\\
\geq& T_{\mu^*}[V_{\mu^{(i)}}\circ \enc](x)-\frac{\gamma+\gamma(1-\gamma)}{1-\gamma}\cdot\max_{a\in\mathcal A}\Delta(\enc,\dyn,\dec,\bar{r},a,x),\\
\geq& T_{\mu^*}[U_{\mu^{(i)}}](z)-\frac{\gamma+\gamma(1-\gamma)+\gamma^2}{1-\gamma}\cdot\max_{a\in\mathcal A}\Delta(\enc,\dyn,\dec,\bar{r},a,x),\\
\end{split}
\]
\end{small}
where the fourth inequality follows from direct algebraic manipulations and the last inequality follows from Lemma \ref{lem:policy_eval} when applied to $U_{\mu^{(i)}}$, i.e., 
\begin{small}
\[
U_{\mu^{(i)}}(x)\leq V_{\pi^{(i)}}\circ \enc(x)+\frac{\gamma}{1-\gamma}\cdot\max_{a\in\mathcal A}\Delta(\enc,\dyn,\dec,\bar{r},a,x)
\]
\end{small}
and from the contraction property of $T_{\mu^*}$, i.e.,
\begin{small}
\[
T_{\mu^*}[U_{\mu^{(i)}}](x)\leq T_{\mu^*}[V_{\pi^{(i)}}\circ \enc](x)+\frac{\gamma^2}{1-\gamma}\cdot\max_{a\in\mathcal A}\Delta(\enc,\dyn,\dec,\bar{r},a,x).
\]
\end{small}
Also notice that with $U^*$ equal to the fixed-point solution of this Bellman operator, we have the following property:
\begin{small}
\[
-\gamma\|U_{\mu^{(i)}}-U^*\|_\infty\leq T_{\mu^*}[U_{\mu^{(i)}}]-T_{\mu^*}[U^*]=T_{\pi^*}[U_{\mu^{(i)}}]-U^*\leq\gamma\|U_{\mu^{(i)}}-U^*\|_\infty.
\]
\end{small}
Furthermore, since $U_{\mu^{(i+1)}}(x)\leq U^*(x)$, we have the chain of inequalities
\begin{small}
\[
\begin{split}
-\|U_{\mu^{(i+1)}}-U^*\|_\infty=U_{\mu^{(i+1)}}(x)-U^*(x)\geq& T_{\mu^*}[U_{\mu^{(i)}}](x)-U^*(x)-\frac{2\gamma}{1-\gamma}\max_{x\in\mathcal X,a\in\mathcal A}\Delta(\enc,\dyn,\dec,\bar{r},a,x)\\
\geq &-\gamma \|U_{\mu^{(i)}}-U^*\|_\infty-\frac{2\gamma}{1-\gamma}\cdot\max_{x\in\mathcal X,a\in\mathcal A}\Delta(\enc,\dyn,\dec,\bar{r},a,x).
\end{split}
\]
\end{small}
In other words, we have
\begin{small}
\[
\|U_{\mu^{(i+1)}}(x)-U^*(x)\|_\infty\leq \gamma \|U_{\mu^{(i)}}-U^*\|_\infty+\frac{2\gamma}{1-\gamma}\max_{x\in\mathcal X,a\in\mathcal A}\Delta(\enc,\dyn,\dec,\bar{r},a,x).
\]
\end{small}
The proof is completed by taking $i\rightarrow\infty$ and noticing that 
\begin{small}
\[
\lim_{i\rightarrow\infty}\|U_{\mu^{(i+1)}}(x)-U^*(x)\|_\infty=\lim_{i\rightarrow\infty}\|U_{\mu^{(i+1)}}(x)-U^*(x)\|_\infty.
\]
\end{small}
\end{proof}

  \section{Value-Guided CARL}
  \label{app:Value-guided CARL}
  Previously the LCE model $(\enc,\dyn,\dec)$ and the latent reward model $\bar{r}$ are learned to minimize the lower bound of approximate policy improvement. While this procedure depends on the current policy $\pi$ and encoder $\enc$ to generate new data for updating these models, the model learning part only consists of the (i) the prediction loss, (ii) consistency loss, and (iii) the policy matching regularization loss between the observation policies $\mu$ and $\pi\circ\enc$ (distillation loss). The LCE model learning objective does not explicitly take into the account of the primary RL objective. 
  
To tackle this issue, we apply the techniques from variational model-based policy optimization \citep{chow2020Variational}, which aims at learning a dynamics model that is also sensitive to the value function of the RL objective, to learn the LCE model. In particular, according to (16) of \cite{chow2020Variational}, the "optimal" observation dynamics model that also takes the value function w.r.t. policy $\mu$ in to the account has the form
\begin{small}
\begin{equation}
\label{eq:optimal-posteriors-closed-form}
P^*(x'|x,a) = \frac{P(x'|x,a)\cdot\exp\big( \tau\cdot\tilde{U}_\mu(x')\big)}{\exp\big(\tau\cdot(\tilde{W}_\mu(x,a) - r(x,a))/\gamma\big)}= P(x'|x,a)\cdot\exp\left(\tau\cdot\frac{r(x,a)+ \gamma\tilde{U}_\mu(x')-\tilde{W}_\mu(x,a) }{\gamma}\right),
\end{equation}
\end{small}
in which $\tilde{U}_\mu(x)$ is the risk-adjusted observation value function at policy $\mu$, i.e.,
\begin{small}
\[
\tilde{U}_\mu(x):=\frac{1}{\tau}\log\mathbb E\left[\exp\left(\tau\cdot\sum_{t=0}^\infty\gamma^tr_\mu(x_t) \right)\mid P_\mu,x_0=x\right],
\]
\end{small}
which is also a unique solution that satisfies the fixed-point property:
\begin{small}
\[
\tilde{U}_\mu(x)=\int_{a}d\mu(a|x)\left[r(x,a)+\gamma\cdot\frac{1}{\tau}\cdot\log\mathbb E_{x'\sim P(\cdot|x,a)}\big[\exp\big(\tau\cdot \tilde{U}(x')\big)\big]\right],
\]
\end{small}
and $\tilde{W}_\mu(x)$ is the corresponding risk-adjusted observation state-action value function at policy $\mu$, i.e.,
\begin{small}
\[
\tilde{W}_\mu(x,a):=r(x,a)+\gamma\cdot\frac{1}{\tau}\cdot\log\mathbb E_{x'\sim P(\cdot|x,a)}\big[\exp\big(\tau\cdot U_\mu(x')\big)\big].
\]
\end{small}
The above value functions are termed "risk-adjusted" because the next state is no longer marginalized by taking an expectation over the original transition probability $P(x'|x,a)$, but instead it is marginalized by taking the exponential risk \citep{ruszczynski2006optimization}, i.e., $\rho_\tau(U(\cdot)|x,a)=\frac{1}{\tau}\log\mathbb E_{x'\sim P(\cdot|x,a)}[\exp(\tau\cdot U(x'))]$, where $\tau$ is the temperature constant.
This modified dynamics model $P^*$ is an exponential twisting of the original transition dynamics $P$ with weight 
\begin{small}
\begin{equation}
w(x,a,x')= \tau\cdot (r(x,a)+\gamma \tilde{U}_\mu(x')-\tilde{W}_\mu(x,a))/\gamma,
\end{equation}
\end{small}
which corresponds to the standard discounted TD-error of the risk-adjusted value functions.

 To incorporate the ``value-guided'' transition model $P^*$ in the LCE model learning, all we need to do is to replace the original transition model $P$ in the prediction loss and in the consistency loss with the value-guided counterpart. Recall the prediction loss:
 \begin{small}
 \[
 L_p(E,F,D,\pi,x)=D_{\text{KL}}\big(P_{\mu}(\cdot|x) \; || \; (D\circ F_\pi\circ E)(\cdot|x)\big)=\int_{x'}dP_{\mu}(x'|x)\log\frac{(D\circ F_\pi\circ E)(x'|x)}{P_{\mu}(x'|x)}.
 \]
 \end{small}
 Since the term $\log P_{\mu}(x'|x)$ is independent to the LCE model, one can equivalent optimize the maximum likelihood (MLE):
 \begin{small}
  \[
 L_{\text{mle}}(E,F,D,\pi,x)=-\int_{x'}dP_{\mu}(x'|x) \cdot
 \log (D\circ F_\pi\circ E)(x'|x).
 \]
 \end{small}
 
 Now with the value-guided transition model, this MLE loss function can be re-written as
 \begin{small}
   \begin{equation}
   \begin{split}
 L_{\text{mle}}(E,F,&\,D,\pi,x)=-\int_{x'}dP^*_{\mu}(x'|x) \cdot
 \log (D\circ F_\pi\circ E)(x'|x)\\
 =&-\int_{x'}dP^*_{\mu}(x'|x) \cdot
 \log (D\circ F_\pi\circ E)(x'|x)\\
 =&-\int_{a}d\mu(a|x)\int_{x'} dP(x'|x,a) \cdot \exp(w(x,a,x'))\cdot
 \log (D\circ F_\pi\circ E)(x'|x).
 \end{split}
 \end{equation}
 \end{small}
 For the consistency loss, consider 
 \begin{small}
 \[
 L_c(E,F,D,\pi,x)=\int_{x'\in\mathcal X}dP_{\mu}(x'|x)\cdot D_\text{KL}\left( \enc(\cdot|x')||(\dyn_\pi\circ\enc)(\cdot|x)\right).
 \]
 \end{small}
 Similar to the derivations of the prediction loss, with the value-guided transition model, this consistency loss function be re-written as
 \begin{small}
 \begin{equation}
 \begin{split}
 L_c(E,F,&D,\pi,x) =\int_{x'\in\mathcal X}dP^*_{\mu}(x'|x)\cdot D_\text{KL}\left( \enc(\cdot|x')||(\dyn_\pi\circ\enc)(\cdot|x)\right)\\
 =&\int_{a}d\mu(a|x)\int_{x'} dP(x'|x,a) \cdot \exp(w(x,a,x'))\cdot D_\text{KL}\left( \enc(\cdot|x')||(\dyn_\pi\circ\enc)(\cdot|x)\right).
 \end{split}
 \end{equation}
 \end{small}

 Below we propose ways to efficiently compute the exponential twisting weight $w(x,a,x')=\tau\cdot (r(x,a)+\gamma \tilde{U}_\mu(x')-\tilde{W}_\mu(x,a))/\gamma$. For simplicity we consider the case when $\tau$ is small, where in this case $\rho_\tau(U(\cdot)|x,a)\approx\mathbb E[U|x,a]$.
 Extending the following arguments requires directly learning the risk-adjusted value function $\tilde U_\mu(x)$ and state-action value functions $\tilde W_\mu(x,a)$, whose details can be found in \cite{borkar2002q}. 
 
 Under this condition, for any $(x,a)\in\mathcal X\times\mathcal A$, we have 
 \begin{small}
 \[
 \tilde{U}_\mu(x)\approx U_\mu(x), \quad\quad \tilde{W}_\mu(x,a)\approx W_\mu(x,a):=r(x,a)+\int_{x'}dP(x'|x,a)U_\mu(x').
 \]
 \end{small}
 
 Instead of computing the value functions in the observation space, we can approximate them with their low-dimensional latent-space counterparts. In particular, Lemma \ref{lem:policy_eval} implies that
 \begin{small}
 \[
 \left|V_{\pi}\circ\enc(x)- U_{\mu}(x)\right|\leq \frac{\gamma}{1-\gamma}\Delta(\enc,\dyn,\dec,\bar{r},\pi,x)+\frac{R_{\max}}{1-\gamma}\cdot D_{\text{KL}}(\pi\circ\enc(\cdot|x)||\mu(\cdot|x)),\,\, \forall x\in\mathcal X.
 \]
 \end{small}
Since we are optimizing the terms on the right side of the bound for the LCE model, 
 if these terms are small, then $U_{\mu}(x)\approx V_{\pi}\circ\enc(x)$. Following analogous derivations we also have the following error bound for the state-action value function: 
 \begin{small}
 \[
 \begin{split}
 &\left|Q_{\pi}\circ\enc(x,a)- W_{\mu}(x,a)\right|\\
 \leq& \frac{\gamma}{1-\gamma}\Delta(\enc,\dyn,\dec,\bar{r},x,a)+\frac{\gamma R_{\max}}{1-\gamma}\cdot D_{\text{KL}}(\pi\circ\enc(\cdot|x)||\mu(\cdot|x)),\,\, \forall x\in\mathcal X,\,\,a\in\mathcal A.
 \end{split}
 \]
 \end{small}
 where the state-action error term is given by
 \begin{small}
 \[
\begin{split}
\Delta(\enc,\dyn,\dec,\bar{r},\,&x,a):=\frac{R_{\max}}{1-\gamma}\sqrt{\frac{-1}{2}\int_{z}d\enc(z|x)\log \dec(x|z)} \\ 
&+ 2\big|r(x,a)-\bar{r}(z,a)\big| + \frac{\gamma R_{\max}}{\sqrt{2}(1-\gamma)}\sqrt{D_{\text{KL}}\big(P(\cdot|x,a) \; || \; (D\circ F\circ E)(\cdot|x,a)\big)} \\
&+ \frac{\gamma R_{\max}}{\sqrt{2}(1-\gamma)}\sqrt{ \int_{x'\in\mathcal X}dP(x'|x,a)\cdot D_\text{KL}\left( \enc(\cdot|x')||(\dyn\circ\enc)(\cdot|x,a)\right)}.
\end{split}
\]
\end{small}
If this terms is small the state-action value function can be approximated by $W_{\mu}(x,a)\approx Q_{\pi}\circ\enc(x,a)$.
 
Finally, recall that we are learning the latent reward model by minimizing the following reward loss: $\big|\int_{z,a}d\enc(z|x)d\pi(a|z)(r(x,a)-\bar{r}(z,a))\big|$.
Therefore, we have that $r(x,a)\approx\bar{r}\circ E(x,a)$.
 Together, the exponential twisting weights can be approximated by the latent reward, latent value function, and latent state-action value function as follows:
 \begin{small}
 \[
 w(x,a,x')\approx \widehat{w}(x,a,x'):=\int_{z,z'}d\enc(z|x)\cdot d\enc(z'|x')\cdot(\bar{r}(z,a)-Q_\pi(z,a))+\gamma V_\pi(z'),
 \]
 \end{small}
and correspondingly the exponential twisting term can be approximated by $\exp(\widehat{w}(x,a,x'))$.

\newpage
\section{CARL Algorithms}
\label{app:CARL Algorithm}
Below in Algorithm \ref{alg:CARL_implementation} we present the practical implementation of the CARL algorithm with notation for all of its variants (offline CARL, online CARL, Value-Guided CARL).

\begin{algorithm}[H]
    \caption{Control Aware Representation Learning (CARL)}
\begin{algorithmic}[1]
    \STATE {\bf Inputs}: A dataset $\mathcal{D}_{\text{real}}$ tuples $(x, a, x^\prime, x_{g})$ from the environment, $Env$. A latent controllable embedding (LCE) network $\mathcal{M}$ consisting of an encoder $E: \mathcal{X}\rightarrow\mathcal{Z}$, transition dynamics $F:\mathcal{Z}\times\mathcal{A}\rightarrow\mathcal{Z}$, decoder $D:\mathcal{Z}\rightarrow\mathcal{X}$, and backwards encoder $B:\mathcal{X}\times\mathcal{Z}\times\mathcal{A}\rightarrow\mathcal{Z}$; plus networks for control: latent critic networks $V_{\phi_1}, V_{\phi_2}: \mathcal{Z} \rightarrow \mathbb{R}$, $Q_{\theta_1}, Q_{\theta_2}: \mathcal{Z} \times \mathcal{A} \rightarrow \mathbb{R}$, and latent actor network $\pi_{\psi}: \mathcal{Z}\rightarrow \mathcal{A}$

    \FOR{$i = 0,\ldots, T$ {\color{gray}\#{\em when $T=0$ this is offline CARL}}} 
        \FOR{$j = 1, \ldots,$ num\_pcc\_epochs}
            \STATE {\color{gray}\#{\em representation learning.}}
            \STATE Train $\mathcal{M}^{(i)}$ using dataset $\mathcal{D}_{\text{real}}$ model 
            \STATE For value-guided CARL, the prediction and consistency loss functions requires the exponential twisting weight $\exp(\frac{\tau}{\gamma}\cdot\hat w(x,a,x'))$, where $
            \hat w(x,a,x')=\int_{z,z'}\enc^{(i-1)}(z|x)\cdot\enc^{(i-1)}(z'|x')\cdot(-||z_{g} - z^\prime||_2-Q_{\bar{\phi}}(z,a))+\gamma V_{\bar{\phi}}(z'))$
        \ENDFOR
        \STATE Initialize a soft actor critic (SAC) policy $\pi^{(i)}$
        \IF {Do policy distillation and $i\ge 1$}
            \STATE {\color{gray} \#{\em Corresponds to the policy distillation loss $L_p$}}
            \FOR {Each policy distillation epoch}
                \STATE $\pi^{(i)}_{\psi} \leftarrow \pi^{(i)}_\psi - \nabla_{\psi}\mathbb{E}_{x\sim D}\left[ D_{KL}\left(\pi^{(i)}_{\psi}\left(E^{(i)}(\cdot|x)\right)||\pi^{(i)}_{\psi}\left(E^{(i-1)}(\cdot|x)\right)\right)\right]$ 
            \ENDFOR
        \ENDIF
        \STATE Initialize a latent space buffer $\mathcal{B}_{latent}$\\
        \STATE {\color{gray}\#{\em learning a latent space policy}}
        \FOR{Each soft actor critic step}
            \STATE Sample real dataset $(x, a, x_{g}) \sim \mathcal{D}_{\text{real}}^{(i)}$\\
            \STATE Generate necessary latent space variables:\\ 
            $z \sim E(\cdot |x), z^\prime \sim F(\cdot|z, u), z_{g} \sim E(\cdot | x_{g}), r = -||z_{g} - z^\prime||_2$
            \STATE Add latent batch to latent buffer $\mathcal{B}_{latent} \leftarrow \mathcal{B}_{latent} \cup (z, a, z^\prime, r, z_{g})$
            \STATE Sample latent buffer $(z, a, z^\prime, r, z_{g}) \sim \mathcal{B}_{latent}$
            \STATE {\color{gray} \#{\em Train the policy $\pi^{(i)}$ with $(z, u, z^\prime, r, z_{g})$ with the SAC algorithm}}
            \STATE $\theta_i \leftarrow \theta_i - \kappa_Q\nabla_{\theta_i}J_Q(\theta_i)$ for $i\in\{1, 2\}$ \hfill{\color{gray} \#{\em Update the Q-function weights}}
            \STATE $\phi_i \leftarrow \phi_i - \kappa_V\nabla_{\phi_i}J_V(\phi_i)$ for $i\in\{1, 2\}$ \hfill{\color{gray} \#{\em Update the V-function weights}}
            \STATE $\psi \leftarrow \psi - \kappa_\pi \nabla_\psi J_\psi(\psi)$ \hfill{\color{gray} \#{\em Update the policy weights}}
            \STATE $\bar{\theta_i} \leftarrow \nu\theta_i +(1-\nu)\bar{\theta_i}$ for $i\in\{1,2\}$ \hfill{\color{gray} \#{\em Update the Q-target critic networks weights}}
            \STATE $\bar{\phi_i} \leftarrow \nu\phi_i +(1-\nu)\bar{\phi_i}$ for $i\in\{1,2\}$ \hfill{\color{gray} \#{\em Update the V-target critic networks weights}}
        \ENDFOR
        \STATE {\color{gray} \#{\em Sample the environment for new real data}}
        \FOR{Each Interaction with Environment}
            \STATE Sample Actions $a \sim \pi^{(i)}(E(\cdot | x))$\\
            \STATE Interact with the environment $x^\prime\leftarrow Env(x,a)$\\
            \STATE Update real data dataset $\mathcal{D}_{\text{real}} \leftarrow \mathcal{D}_{\text{real}} \cup (x,a,x^\prime, x_{g})$
            \STATE Update current state $x \leftarrow x^\prime$
        \ENDFOR
    \ENDFOR
\end{algorithmic}
\label{alg:CARL_implementation}
\end{algorithm}

{\bf Soft Actor Critic (SAC) Updates}
The policy parameters $\psi$ are optimized to update the latent space policy towards the exponential of the soft Q-function,

\begin{equation}
    J_{\pi}(\psi) = \underset{z_t \sim\mathcal{D}}{\mathbb{E}}\left[\underset{a_t\sim\pi_\psi}{\mathbb{E}}\left[\alpha log(\pi_{\psi}(a_t|z_t) -Q_{\theta}(z_t, a_t)\right]\right]
\end{equation}

Our updates to the Q network minimize the following loss function:

\begin{equation}
    J_Q(\theta) = \mathbb{E}_{(z_t, a_t)\sim\mathcal{D}} \left[\frac{1}{2}\left(Q_{\theta}(z_t, a_t) - \hat{Q}(z_t, a_t)\right)^2\right]
\end{equation}

where:

\begin{equation}
    \label{eq:SAC Q Update}
    \hat{Q}_\theta(z_t, a_t) = r(z_t, a_t) + \sum_{i=0}^{k-1}\gamma^{i+1}r(z_{t+i+1}, a_{t+i+1}) | a_{t} \sim \pi_{\psi}(\cdot | z_t), z_{t+1} \sim F(\cdot | z_t, a_t)
\end{equation}

here, $F$ is the learned latent space transition model and $r(z_t, a_t) = ||z_{goal} - z_{t+1}||^2_2$ where, $z_{t+1} \sim F(\cdot| z_t, a_t)$ and $z_{goal} \sim E(\cdot|x_{goal})$ and $x_{goal}$ is the observation of the environment; additionally, $k$ is a tunable hyperparameter of the number of rollouts in the latent space should we rollout our model to sufficiently approximate the Q-value. As in \cite{pmlr-v80-haarnoja18b} we utilize two Q-functions and take the minimum of the Q-functions to generate the value in the actor loss function. We note that we don't have value network updates as we tried adding value networks but were unable to get good results.

\section{Experimental Details}
\label{app:Experimental Details}
In the following section we will provide a description of the domains and implementation details used for the experiments. For all the experiments we define the reward function as $r(z, a) = -(z-z_{g})^\top Q(z-z_{g}) - a^{\top}Ra$, where $z$ and $z_{g}$ are latent states of the current and goal observation, and  $Q=\kappa\cdot I_{n_{z}}, R=I_{n_a}$ where $\kappa=50$, are respectively penalty weights on the state error and action. This reward configuration follows exactly from \cite{Levine2020Prediction}.

\begin{paragraph}{Planar System}
In this task the main goal is to navigate an agent in a surrounded area on a 2D plane \citep{breivik2005principles}, whose goal is to navigate from a corner to the opposite one, while avoiding the six obstacles in this area. The system is observed through a set of $40 \times 40$ pixel images taken from the top view, which specifies the agent's location in the area. Actions are two-dimensional and specify the $x-y$ direction of the agent's movement, and given these actions the next position of the agent is generated by a deterministic underlying (unobservable) state evolution function. \textit{Start State:} top-left corner. \textit{Goal State:} one of three corners (excluding top-left corner). \textit{Agent's Objective:} agent is within Euclidean distance of $5$ from the goal state. 

For the biased variant of this experiment we uniformly sample a proportion, $p$, of the total samples within a $30 \times 30$ pixel region, which doesn't include any of the goal states, and the other $1-p$ proportion of the samples are sampled uniformly from the entire underlying state space.

\end{paragraph}

\begin{paragraph}{Inverted Pendulum -- SwingUp} This is the classic problem of controlling an inverted
pendulum \citep{furuta1991swing} from $48 \times 48$ pixel images. The goal of this task is to swing up an under-actuated pendulum from the downward resting position (pendulum hanging down) to the top position and to balance it. The underlying state $\state_t$ of the system has two
dimensions: angle and angular velocity, which is unobservable. The control (action) is $1$-dimensional, which is the torque applied to the joint of the pendulum. For all PCC based algorithms, we opt to consider each observation $\obs_t$ as two images generated from consecutive time-frames (the current time and the previous time; this was also done in the original PCC paper \citep{Levine2020Prediction}. This is because each image only shows the position of the pendulum and does not contain any information about the velocity. \textit{Start State:} Pole is resting down, \textit{Agent's Objective:} pole's angle is within $\pm \pi/6$ from an upright position.

For the biased experimentation of this experiment we sample a proportion, $p$, of the total samples from when the pendulum is in it's closer to its resting position $[-\pi, -2.0]\cup[2, \pi]$ and the other $1-p$ samples when the pendulum is within $\pm 0.5$ from an upright position. 

\end{paragraph}

\begin{paragraph}{CartPole} This is the visual version of the classic task of controlling a cart-pole system \citep{geva1993cartpole}. The goal in this task
is to balance a pole on a moving cart, while the cart avoids hitting the left and right boundaries. The control (action) is $1$-dimensional, which is the force applied
to the cart. The underlying state of the system $\state_t$ is $4$-
dimensional, which indicates the angle and angular velocity of the pole, as well as the position and velocity of the cart. Similar to the inverted pendulum, in order to maintain the Markovian property the observation $\obs_t$ is a stack of two
$80\times 80$ pixel images generated from consecutive time-frames.
\textit{Start State:} Pole is randomly sampled in $\pm \pi/6$. \textit{Agent's Objective:} pole's angle is within $\pm \pi/10$ from an upright position.

For the biased experiment we sample a proportion, $p$, of the total samples from when the pole's angle is sampled from $[-\pi/6, -\pi/10]\cup[\pi/10, \pi/6]$ and the other $1-p$ samples are sampled as before uniformly from the given state space.

\end{paragraph}

\begin{paragraph}{$3$-link Manipulator --- SwingUp \& Balance} The goal in this task is to move a $3$-link manipulator from the initial position (which is the downward resting position) to a final position (which is the top position) and balance it. In the $1$-link case, this experiment is reduced to inverted pendulum. In the $2$-link case the setup is similar to that of arcobot \citep{spong1995swing}, except that we have torques applied to all intermediate joints, and in the $3$-link case the setup is similar to that of the $3$-link planar
robot arm domain that was used in the E2C paper, except that the robotic arms are modeled by simple rectangular rods (instead of real images of robot arms), and our task success criterion requires both swing-up (manipulate to final position) and balance.\footnote{Unfortunately due to copyright issues, we cannot test our algorithms on the original $3$-link planar
robot arm domain.} The underlying (unobservable) state $\state_t$ of the system is $2N$-dimensional, which indicates the relative angle and angular velocity at each link, and the actions are $N$-dimensional,
representing the force applied to each joint of the arm. The state evolution is modeled by the standard Euler-Lagrange equations \citep{spong1995swing, lai2015singularity}.
Similar to the inverted pendulum and cartpole, in order to maintain the Markovian property, the observation state $\obs_t$ is a stack of two $80\times 80$ pixel images of the $N$-link manipulator generated from consecutive time-frames. In the experiments we will evaluate the models based on the case of $N=2$ ($2$-link manipulator) and $N=3$ ($3$-link manipulator).
\textit{Start State:} $1^{\text{st}}$ pole with angle $\pi$, $2^{\text{nd}}$ pole with angle $2\pi/3$, and $3^{\text{rd}}$ pole with angle $\pi/3$, where angle $\pi$ is a resting position. \textit{Agent's Objective:} the sum of all poles' angles is within $\pm \pi/6$ from an upright position.

For the biased experiment we sample a proportion, $p$ of the total samples of when the 1st pole is within $\pm \pi/2$, the 2nd pole is within angle $\pm\pi/3$, and the 3rd pole is within angle $\pm\pi/6$ of the upright position and the other $1-p$ samples are sampled as before uniformly from the given state space.
\end{paragraph}

\subsection{Data Generation Procedure}
For all algorithms that use the PCC framework for representation learning, we always start by sampling triplets of the form $(x_t, a_t, x_{t+1})$, which is done by (1) uniformly randomly sampling an underlying state $s_t$ from the environment and creating the corresponding observation $x_t$, (2) uniformly randomly sampling a valid action $a_t$, and (3) obtaining the next state $s_{t+1}$ through the environment's true dynamics and creating the corresponding observation $x_{t+1}$. 

When interacting with the true underlying MDP, sampling the environment for more data for the iterative online variant of our algorithm, at iteration $i$ of our algorithm, we are following our learned policy $\pi^{(i)}$. We start with an initial observation $x_0$ and generate our initial action $a_0$, $a_0 \sim \pi^{(i)}_\psi(E(\cdot | x_0))$ and continue following our learned policy $\pi^{(i)}$ to get our action $a_j$, $a_j \sim \pi^{(i)}_\psi(E(\cdot|x_j))$.We continue this process until we have reached the end of the episode or the pre-defined number of samples we draw from the environment. 

In SOLAR each training sample is an episode $\{x_1, a_1, x_2, \cdots, x_T, a_T, x_{T+1}\}$ where $T$ is the control horizon. We uniformly sample $T$ actions from the action space, apply the dynamics $T$ times, and generate the $T$ corresponding observations. 

\subsection{Implementation}\label{sec:nn_arch}

In the following we describe architectures and hyper-parameters that were used for training the different algorithms.

\subsubsection{Training Hyper-Parameters and Regulizers}
SOLAR training specifics, we used their default setting:

\begin{itemize}
    \item Batch size of 2.
    \item ADAM \citep{kingma2014adam} with $\beta_1=0.9$, $\beta_2=0.999$, and $\epsilon=10^{-8}$. We also use a learning rate $\alpha_{model}=2\cdot10^{-5}\times \textrm{horizon}$ for the learning rate $\mathcal{M}\mathcal{N}\mathcal{I}\mathcal{W}$ prior and $\alpha=10^{-3}$ for other parameters.
    \item $\beta_{start}, \beta_{end}, \beta_{rate}=(10^{-4}, 10.0, 5\cdot 10^{-5})$
    \item Local inference and control:
        \begin{itemize}
            \item Data strength: 50
            \item KL step: 2.0
            \item Number of Iterations: 10
        \end{itemize}
\end{itemize}


PCC training specifics, we use their reported optimal hyperparameters:
\begin{itemize}
    \item Batch size of $128$.
    \item ADAM with $\alpha = 5 \cdot 10^{-4}$, $\beta_1 = 0.9, \beta_2 = 0.999$, and $\epsilon = 10^{-8}$.
    \item L2 regularization with a coefficient of $10^{-3}$.
    \item $(\lambda_p, \lambda_c, \lambda_{\text{cur}}) = (1,7,1)$, and the additive Gaussian noise in the curvature loss is $\mathcal{N}(0, \sigma^2)$, where $\sigma^2 = 0.01$.
    \item Additional VAE~\citep{kingma2013auto} loss term $\ell_{\text{VAE}} = -\mathbb{E}_{q(z|x)}[\log p(x|z)] + D_{\text{KL}}(q(z|x)||p(z))$ with a very small coefficient of $0.01$, where $p(z) = \mathcal{N}(0,1)$.
    \item Additional deterministic reconstruction loss with coefficient $0.3$: given the current observation $x$, we take the means of the encoder output and the dynamics model output, and decode to get the reconstruction of the next observation.
\end{itemize}

CARL training specifics:
\begin{itemize}
    \item Batch size of 128.
    \item ADAM with $\alpha = 5 \cdot 10^{-4}$, $\beta_1 = 0.9, \beta_2 = 0.999$, and $\epsilon = 10^{-8}$.
    \item L2 regularization with a coefficient of $10^{-3}$.
    \item The additive Gaussian noise in the curvature loss is $\mathcal{N}(0, \sigma^2)$, where $\sigma^2 = 0.01$.
    \item As in \cite{Levine2020Prediction} we use the deterministic reconstruction loss with coefficient $0.3$.
\end{itemize}

\subsubsection{Network Architectures}
We next present the specific architecture choices for each domain. For fair comparison, The numbers of layers and neurons of each component were shared across all algorithms. ReLU non-linearities were used between each two layers.

\textbf{Encoder:} composed of a backbone (either a MLP or a CNN, depending on the domain) and an additional fully-connected layer that outputs mean variance vectors that induce a diagonal Gaussian distribution (for PCC, SOLAR, and all CARL variants).

\textbf{Decoder:} composed of a backbone (either a MLP or a CNN, depending on the domain) and an additional fully-connected layer that outputs logits that induce a Bernoulli distribution (for PCC, SOLAR, and all CARL variants)

\textbf{Dynamical model:} the path that leads from $\{ \latent_t, a_t \}$ to  $\hat{\latent}_{t+1}$. Composed of a MLP backbone and an additional fully-connected layer that outputs mean and variance vectors that induce a diagonal Gaussian distribution (for PCC, SOLAR, and all CARL variants).

\textbf{Backwards dynamical model:} the path that leads from $\{\hat{z}_{t+1}, a_t, x_t\}$ to $z_t$. Each of these inputs goes to fully-connected layer of $\{N_z, N_u, N_x\}$ neurons respectively. These outputs are then concatenated and passed through another layer of $N_{joint}$ neurons, and finally with an additionally fully-connected layer we output the mean and variance vectors that induce a diagonal Gaussian distribution.

\textbf{SAC Architecture:} For all of our environments with all CARL algorithms, we utilized the same SAC architecture as seen in Table \ref{tab:SAC Params}:

\begin{table}[h]
\centering
\begin{tabular}{|l|c|}
\hline
\textbf{Hyper Parameters for SAC} & \textbf{Value(s)} \\ [0.5ex]
\hline
Discount Factor & 0.99 \\
\hline
Critic Network Architecture & MLP with 2 hidden layers of size 256 \\
\hline
Actor Network Architecture & MLP with 2 hidden layers of size 256 \\
\hline
Exploration policy & $\mathcal{N}(0, \sigma=1)$ \\
\hline
Exploration noise ($\sigma$) decay & 0.999 \\
\hline
Exploration noise ($\sigma$) minimum & 0.025 \\
\hline
Temperature & 0.99995 \\
\hline
Soft target update rate ($\tau$) & 0.005 \\
\hline
Replay memory size & $10^6$ \\
\hline
Minibatch size & 128 \\
\hline
Number of Rollouts in the Latent space, $k$ in (\ref{eq:SAC Q Update}) & 5 \\
\hline
Critic learning rate & 0.001 \\
\hline
Actor learning rate & 0.0005 \\
\hline
Neural network optimizer & Adam \\
\hline
\end{tabular}
\caption{Hyper parameters for the SAC controller.}
\label{tab:SAC Params}
\end{table}

\begin{paragraph}{Planar system}
\begin{itemize}
    \item \textbf{Input:} $40\times 40$ images.
    \item \textbf{Actions space:} $2$-dimensional
    \item \textbf{Latent space:} $2$-dimensional
    \item \textbf{Encoder:} $3$ Layers: $300$ units - $300$ units - $4$ units ($2$ for mean and $2$ for variance)
    \item \textbf{Decoder:} $3$ Layers: $300$ units - $300$ units - $1600$ units (logits)
    \item \textbf{Dynamics:} $3$ Layers: $20$ units - $20$ units - $4$ units
    \item \textbf{Backwards dynamics:} $N_\latent=5, N_a=5, N_\obs=100$ - $N_{\text{joint}}=100$ - $4$ units
    \item \textbf{Number of control actions:} or the planning horizon $T = 40$
    \item \textbf{Offline and Online CARL hyperparameters: } $\lambda_{ed} = 0.01, \lambda_{p} = 1, \lambda_{c} = 7, \lambda_{cur} = 1$
    \item \textbf{Value-Guided CARL hyperparameters: } $\lambda_{ed} = 0.01, \lambda_{p} = 2, \lambda_{c} = 11, \lambda_{cur} = 1, \tau = 1/30.0$
    \item \textbf{Proportion of biased samples:} $p = 0.5$
    \item \textbf{Number of samples from the environment per iteration $i$ in Algorithm  \ref{alg:CARL_implementation}}: 128
    \item \textbf{Initial standard deviation for collecting data (SOLAR):} 1.5 for both global and local training.
\end{itemize}
\end{paragraph}

\begin{paragraph}{Inverted Pendulum -- SwingUp}
\begin{itemize}
    \item \textbf{Input:} Two $48 \times 48$ images. 
    \item \textbf{Actions space:} $1$-dimensional
    \item \textbf{Latent space:} $3$-dimensional
    \item \textbf{Encoder:} $3$ Layers: $500$ units - $500$ units - $6$ units ($3$ for mean and $3$ for variance)
    \item \textbf{Decoder:} $3$ Layers: $500$ units - $500$ units - $4608$ units (logits)
    \item \textbf{Dynamics:} $3$ Layers: $30$ units - $30$ units - $6$ units
    \item \textbf{Backwards dynamics:} $N_\latent=10, N_a=10, N_\obs=200$ - $N_{\text{joint}}=200$ - $6$ units
    \item \textbf{Number of control actions:} or the planning horizon $T = 400$ 
    \item \textbf{Offline and Online CARL environment hyperparameters:} $\lambda_{ed} = 0.01, \lambda_{p} = 1, \lambda_{c} = 11, \lambda_{cur} = 1$
    \item \textbf{Value-Guided CARL environment hyperparameters: } $\lambda_{ed} = 0.01, \lambda_{p} = 1, \lambda_{c} = 7, \lambda_{cur} = 1, \tau = 1/60.0$
    \item \textbf{Proportion of biased samples:} $p = 0.95$
    \item \textbf{Number of samples from the environment per iteration $i$ in Algorithm  \ref{alg:CARL_implementation}}: 128
    \item \textbf{Initial standard deviation for collecting data (SOLAR):} 0.5 for both global and local training.

\end{itemize}
\end{paragraph}

\begin{paragraph}{Cart-pole -- Balancing}
\begin{itemize}
    \item \textbf{Input:} Two $80 \times 80$ images. 
    \item \textbf{Actions space:} $1$-dimensional
    \item \textbf{Latent space:} $8$-dimensional
    \item \textbf{Encoder:} $6$ Layers: Convolutional layer: $32 \times 5 \times 5$; stride $(1,1)$ - Convolutional layer: $32 \times 5 \times 5$; stride $(2,2)$ - Convolutional layer: $32 \times 5 \times 5$; stride $(2,2)$ - Convolutional layer: $10 \times 5 \times 5$; stride $(2,2)$ - $200$ units - $16$ units ($8$ for mean and $8$ for variance)
    \item \textbf{Decoder:} $6$ Layers: $200$ units - $1000$ units - $100$ units - Convolutional layer: $32 \times 5 \times 5$; stride $(1,1)$ - Upsampling $(2,2)$ - convolutional layer: $32 \times 5 \times 5$; stride $(1,1)$ - Upsampling $(2,2)$ - Convolutional layer: $32 \times 5 \times 5$; stride $(1,1)$ - Upsampling $(2,2)$ - Convolutional layer: $2 \times 5 \times 5$; stride $(1,1)$
    \item \textbf{Dynamics:} $3$ Layers: $40$ units - $40$ units - $16$ units
    \item \textbf{Backwards dynamics:} $N_\latent=10, N_a=10, N_\obs=300$ - $N_{\text{joint}}=300$ - $16$ units
    \item \textbf{Number of control actions:} or the planning horizon $T = 200$
    \item \textbf{Offline and Online CARL environment hyperparameters:} $\lambda_{ed} = 0.01, \lambda_{p} = 1, \lambda_{c} = 7, \lambda_{cur} = 1$
    \item \textbf{Value-Guided CARL environment hyperparameters: } $\lambda_{ed} = 0.01, \lambda_{p} = 1, \lambda_{c} = 7, \lambda_{cur} = 1, \tau = 1/40.0$
    \item \textbf{Proportion of biased samples:} $p = 0.8$
    \item \textbf{Number of samples from the environment per iteration $i$ in Algorithm  \ref{alg:CARL_implementation}}: 256
    \item \textbf{Initial standard deviation for collecting data (SOLAR):} 10 for global and 5 for local training.

\end{itemize}
\end{paragraph}

\begin{paragraph}{$3$-link Manipulator --- Swing Up \& Balance}
\begin{itemize}
    \item \textbf{Input:} Two $80 \times 80$ images.
    \item \textbf{Actions space:} $3$-dimensional
    \item \textbf{Latent space:} $8$-dimensional
    \item \textbf{Encoder:} $6$ Layers: Convolutional layer: $32 \times 5 \times 5$; stride $(1,1)$ - Convolutional layer: $32 \times 5 \times 5$; stride $(2,2)$ - Convolutional layer: $32 \times 5 \times 5$; stride $(2,2)$ - Convolutional layer: $10 \times 5 \times 5$; stride $(2,2)$ - $500$ units - $16$ units ($8$ for mean and $8$ for variance)
    \item \textbf{Decoder:} $6$ Layers: $200$ units - $1000$ units - $100$ units - Convolutional layer: $32 \times 5 \times 5$; stride $(1,1)$ - Upsampling $(2,2)$ - convolutional layer: $32 \times 5 \times 5$; stride $(1,1)$ - Upsampling $(2,2)$ - Convolutional layer: $32 \times 5 \times 5$; stride $(1,1)$ - Upsampling $(2,2)$ - Convolutional layer: $2 \times 5 \times 5$; stride $(1,1)$
    \item \textbf{Dynamics:} $3$ Layers: $40$ units - $40$ units - $16$ units
    \item \textbf{Backwards dynamics:} $N_\latent=10, N_a=10, N_\obs=400$ - $N_{\text{joint}}=400$ - $16$ units
    \item \textbf{Number of control actions:} or the planning horizon $T = 200$
    \item \textbf{Offline and Online CARL environment hyperparameters:} $\lambda_{ed}= 0.01, \lambda_{p} = 1, \lambda_{c} = 11, \lambda_{cur} = 1$
    \item \textbf{Value-Guided CARL environment hyperparameters: } $\lambda_{ed} = 0.01, \lambda_{p} = 2, \lambda_{c} = 11, \lambda_{cur} = 1, \tau = 1/60.0$
    \item \textbf{Proportion of biased samples:} $p = 0.2$
    \item \textbf{Number of samples from the environment per iteration $i$ in Algorithm  \ref{alg:CARL_implementation}}: 128
    \item \textbf{Initial standard deviation for collecting data (SOLAR):} 1 for global and 0.5 for local training.

\end{itemize}
\end{paragraph}

\section{Additional Experiments}
\label{app:Additional Experiments}
\textbf{Policy Distillation}
\label{app:Policy Distill}
In our iterative algorithm, we describe a method to connect policies from two different latent spaces in eq. (\ref{eq:policy-distillation}). In Figure \ref{fig:train_curve} we show the learning curves for Online CARL with and without policy distillation. In general, when utilizing policy distillation, we achieve similar performance to the iterative variant of our algorithm. Additionally, these results show that with policy distillation, in the three-pole and swingup tasks we are able to achieve faster convergence. Another observed added benefit is that with policy distillation we achieve more stability in the final metrics as we add more samples form our environment across environments. 

\begin{figure}
    \centering
    \begin{subfigure}[b]{0.48\textwidth}
        \centering
        \includegraphics[width=\linewidth]{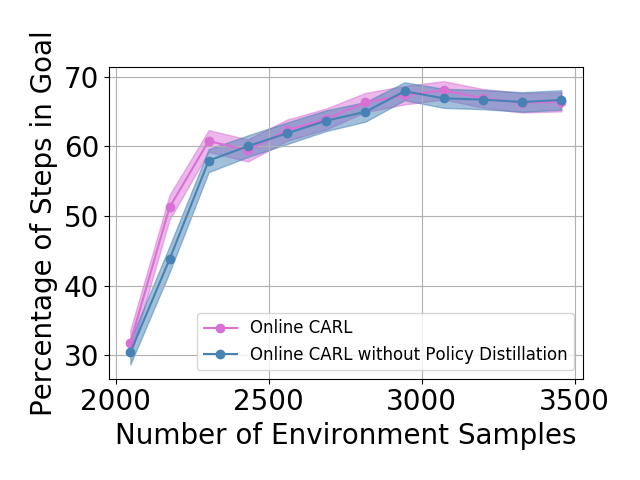}
        \subcaption{Planar}
    \end{subfigure}
    \begin{subfigure}[b]{0.48\textwidth}
        \centering
        \includegraphics[width=\linewidth]{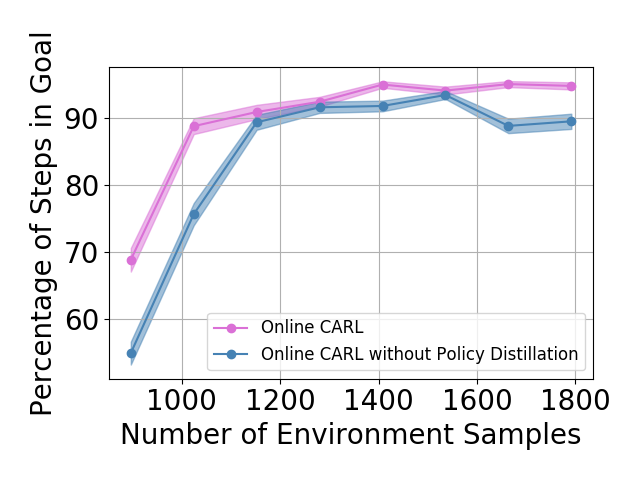}
        \caption{Swingup}
    \end{subfigure}
    \begin{subfigure}[b]{0.48\textwidth}
        \centering
        \includegraphics[width=\linewidth]{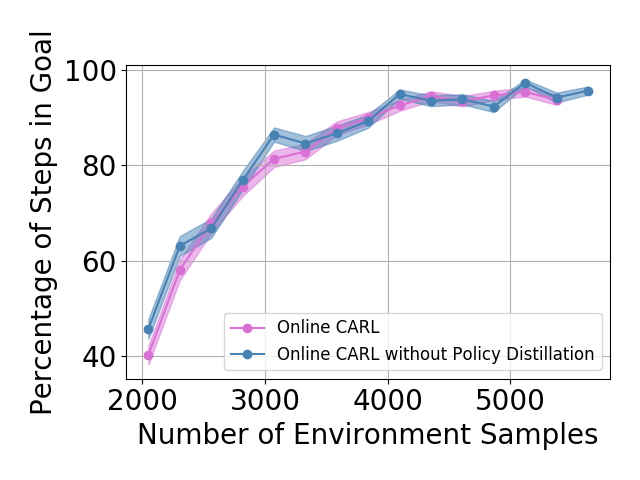}
        \caption{Cartpole}
    \end{subfigure}
    \begin{subfigure}[b]{0.48\textwidth}
        \centering
        \includegraphics[width=\linewidth]{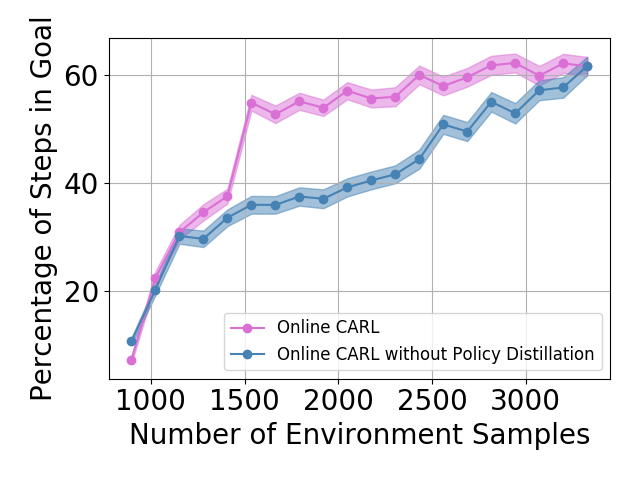}
        \subcaption{Three-pole}
    \end{subfigure}
    \caption{Training curves comparing our online algorithm with and without policy distillation on continuous control environments. The solid curves depict the mean of the experiments and the standard deviations correspond to the standard deviation of the means.}
    \label{fig:train_curve}
\end{figure}

\textbf{Latent Representation and Performance for the Planar System}
\label{app:latent rep}
All of the following figures were trained using the PCC framework. We present 5 representations with the worst control performance (Figure \ref{fig:bad_planar_ilqr}) and 5 representations that had the best control performance (Figure \ref{fig:good_planar_ilqr}), with the PCC algorithm; thus we were using iLQR as the controller. Additionally, we present 5 representations that performed the worst (Figure \ref{fig:bad_planar_sac}) and 5 representations that performed the best (Figure \ref{fig:good_planar_sac}) with our offline CARL algorithm; thus, we were using SAC as the controller. These maps were generated by uniformly sampling a state $s$ from the underlying environment, creating a corresponding observation $x$, and using the encoder create the latent representation $z = \mathbb{E}[E(\cdot |x)]$. All of the latent maps presented in Figures 6-9 are generated from the same PCC framework and same hyperparameters so the best performing maps are all fairly similar and the worst performing maps have similarities. It is important to note that even though these latent maps are similar, it is clear that their is a large difference in performance in table 3. Importantly, it is clear that iLQR struggles significantly more than SAC in these non-linear environments as seen in the worst case performance and the corresponding latent maps, where the latent maps contain additional twisting or curvature resulting in poorer performance.

In this case it is obvious that control in several of the latent representations that performed poorly would be difficult as there are regions that are highly non-smooth, non-locally-linear; thus, a locally linear controller such as iLQR is likely to perform poorly. We compare the top and worst 5 representations trained using the PCC framework, with the only difference being the controller (SAC vs. iLQR) in table \ref{tab:top_bottom_resu}. 

\begin{figure}[h]
    \centering
    \begin{subfigure}[b]{0.19\textwidth}
        \centering
        \includegraphics[width=\linewidth]{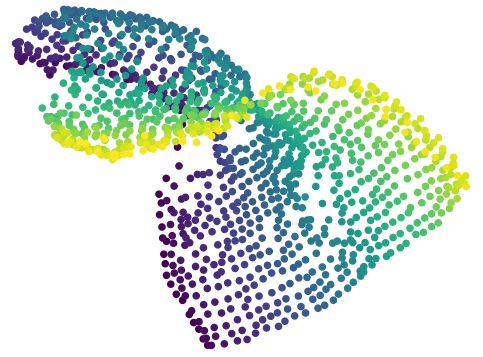}
    \end{subfigure}
    \begin{subfigure}[b]{0.19\textwidth}
        \centering
        \includegraphics[width=\linewidth]{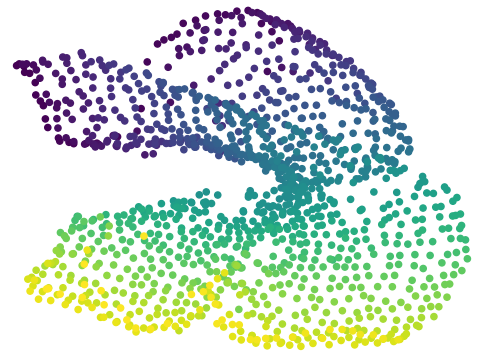}
    \end{subfigure}
    \begin{subfigure}[b]{0.19\textwidth}
        \centering
        \includegraphics[width=\linewidth]{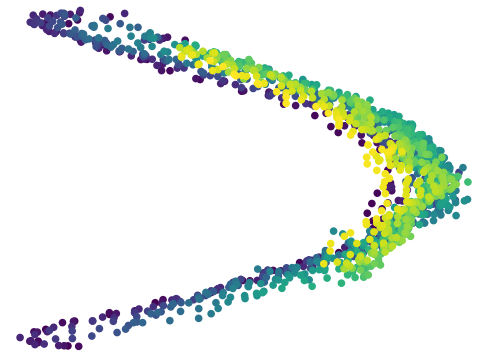}
    \end{subfigure}
    \begin{subfigure}[b]{0.19\textwidth}
        \centering
        \includegraphics[width=\linewidth]{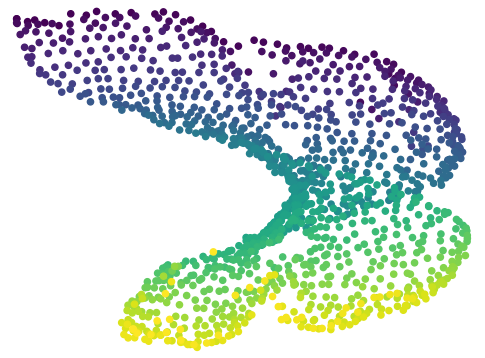}
    \end{subfigure}
    \begin{subfigure}[b]{0.19\textwidth}
        \centering
        \includegraphics[width=\linewidth]{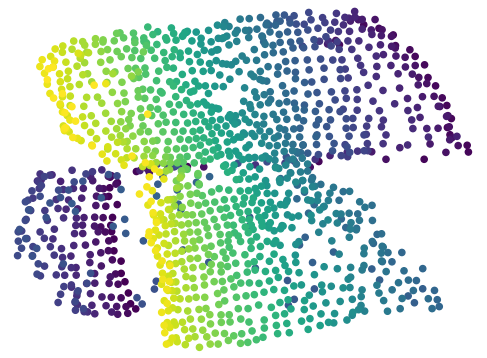}
    \end{subfigure}
    \caption{Latent maps for the 5 worst performing representations on average with PCC as the algorithm; thus, latent control with iLQR as the controller.}
    \label{fig:bad_planar_ilqr}
\end{figure}

\begin{figure}[h]
    \centering
    \begin{subfigure}[b]{0.19\textwidth}
        \centering
        \includegraphics[width=\linewidth]{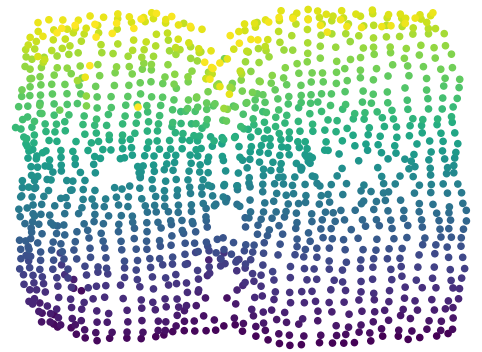}
    \end{subfigure}
    \begin{subfigure}[b]{0.19\textwidth}
        \centering
        \includegraphics[width=\linewidth]{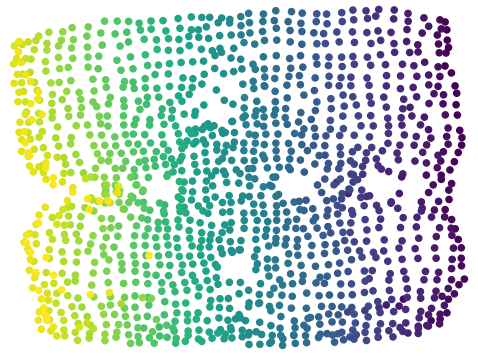}
    \end{subfigure}
    \begin{subfigure}[b]{0.19\textwidth}
        \centering
        \includegraphics[width=\linewidth]{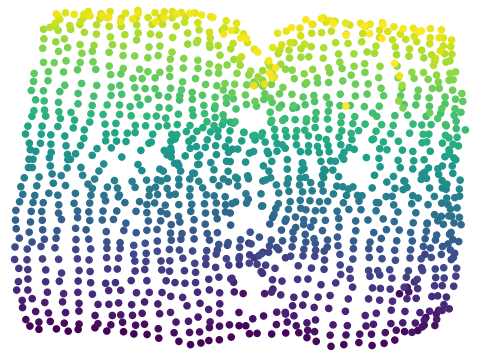}
    \end{subfigure}
    \begin{subfigure}[b]{0.19\textwidth}
        \centering
        \includegraphics[width=\linewidth]{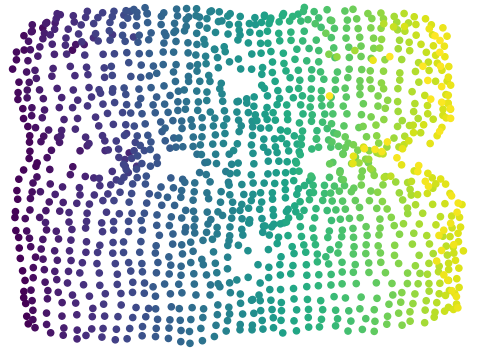}
    \end{subfigure}
    \begin{subfigure}[b]{0.19\textwidth}
        \centering
        \includegraphics[width=\linewidth]{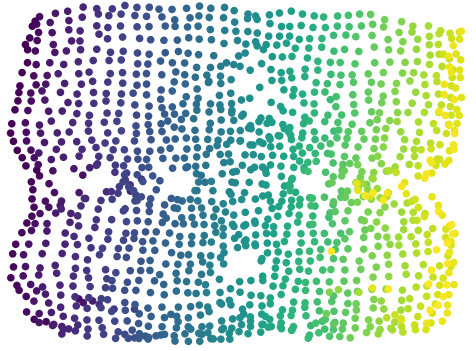}
    \end{subfigure}
    \caption{Latent maps for the 5 best performing representations on average with PCC as the algorithm; thus, latent control with iLQR as the controller.}
    \label{fig:good_planar_ilqr}
\end{figure}

\begin{figure}[h]
    \centering
    \begin{subfigure}[b]{0.19\textwidth}
        \centering
        \includegraphics[width=\linewidth]{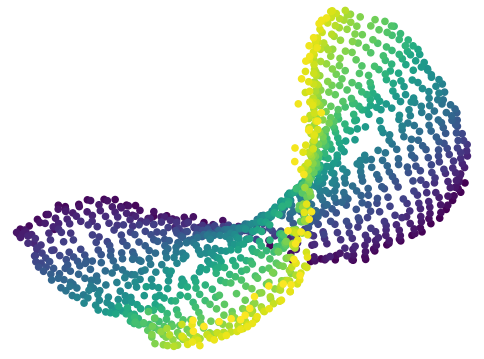}
    \end{subfigure}
    \begin{subfigure}[b]{0.19\textwidth}
        \centering
        \includegraphics[width=\linewidth]{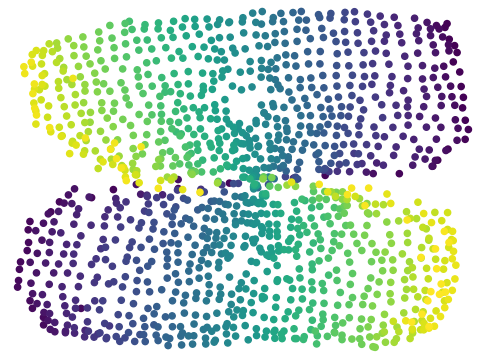}
    \end{subfigure}
    \begin{subfigure}[b]{0.19\textwidth}
        \centering
        \includegraphics[width=\linewidth]{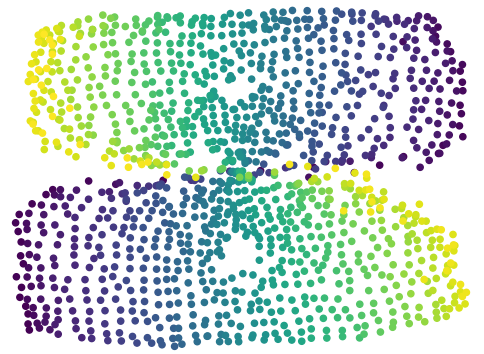}
    \end{subfigure}
    \begin{subfigure}[b]{0.19\textwidth}
        \centering
        \includegraphics[width=\linewidth]{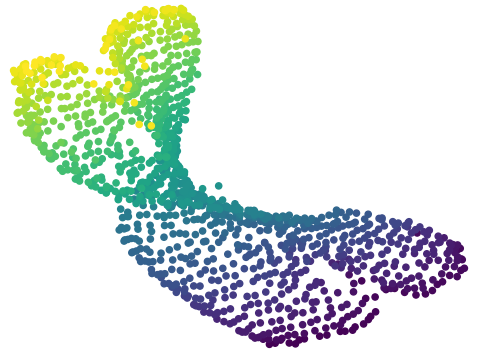}
    \end{subfigure}
    \begin{subfigure}[b]{0.19\textwidth}
        \centering
        \includegraphics[width=\linewidth]{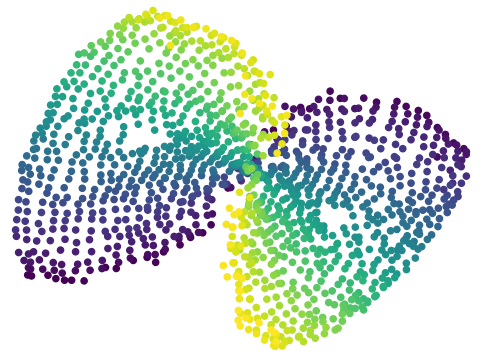}
    \end{subfigure}
    \caption{Latent maps for the 5 worst performing representations on average with Offline CARL as the algorithm; thus, latent control with SAC as the controller.}
    \label{fig:bad_planar_sac}
\end{figure}

\begin{figure}[h]
    \centering
    \begin{subfigure}[b]{0.19\textwidth}
        \centering
        \includegraphics[width=\linewidth]{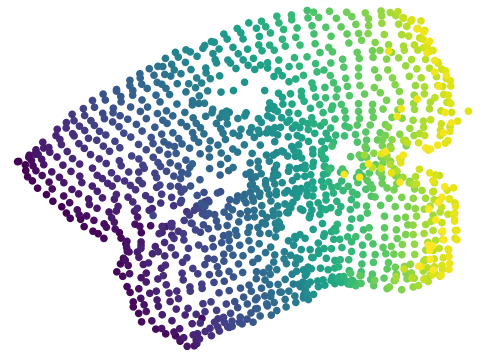}
    \end{subfigure}
    \begin{subfigure}[b]{0.19\textwidth}
        \centering
        \includegraphics[width=\linewidth]{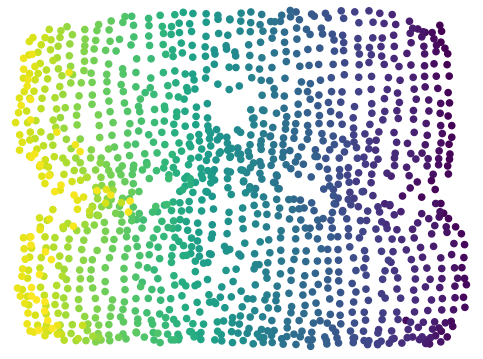}
    \end{subfigure}
    \begin{subfigure}[b]{0.19\textwidth}
        \centering
        \includegraphics[width=\linewidth]{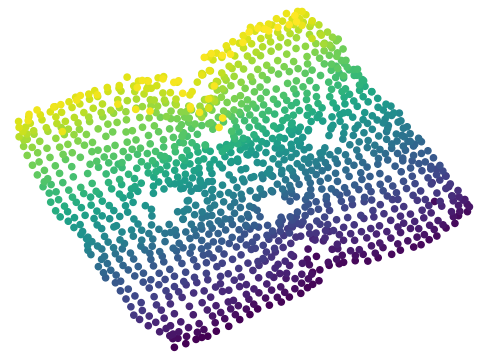}
    \end{subfigure}
    \begin{subfigure}[b]{0.19\textwidth}
        \centering
        \includegraphics[width=\linewidth]{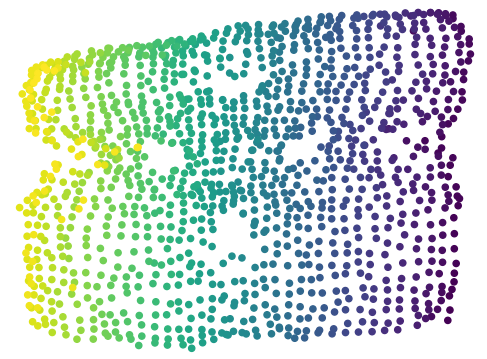}
    \end{subfigure}
    \begin{subfigure}[b]{0.19\textwidth}
        \centering
        \includegraphics[width=\linewidth]{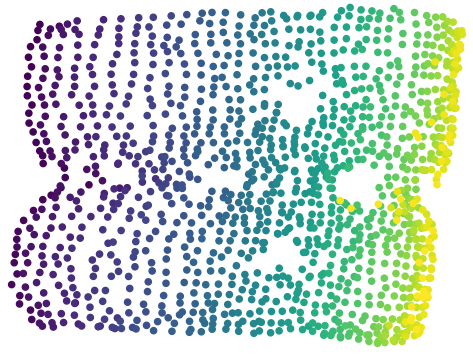}
    \end{subfigure}
    \caption{Latent maps for the 5 best performing representations on average with PCC as the algorithm; thus, latent control with SAC as the controller.}
    \label{fig:good_planar_sac}
\end{figure}

\begin{table}[h]
\centering
\begin{tabular}{|l | l | l | l |}
\hline
\textbf{Environment} & \textbf{Algorithm} & \textbf{Worst 5 Avg. Results} & \textbf{Top 5 Avg. Results} \\
\hline
Planar & PCC & $6.15 \pm 2.89$ & $62.25\pm 4.45$\\
Planar & Offline CARL  & $\textbf{33.46} \pm \textbf{4.61}$ & $\textbf{78.87}\pm \textbf{0.19}$\\
\hline\hline
Swingup & PCC & $\textbf{68.07} \pm \textbf{3.49}$ & $95.71\pm 0.38$\\
Swingup & Offline CARL & $57.22 \pm 3.71$ & $\textbf{98.50}\pm \textbf{0.0}$\\
\hline\hline
Cartpole & PCC & $50.98\pm 5.44$ & $99.85\pm 0.08$\\
Cartpole & Offline CARL & $\textbf{74.44} \pm \textbf{5.28}$ & $\textbf{100.0} \pm \textbf{0.0}$\\
\hline\hline
Three-pole & PCC & $0\pm 0$ & $18.42\pm 2.98$\\
Three-pole & Offline CARL & $\textbf{6.17}\pm \textbf{1.71}$ & $\textbf{85.77} \pm \textbf{0.23}$\\
\hline

\end{tabular}
\caption{Percentage of steps in goal state; averaged over the 5 worst models and the 5 best models.}
\label{tab:top_bottom_resu}
\end{table}

\newpage
\newpage

\textbf{Additional SOLAR Results}
\label{sec:Additional SOLAR Results}
In our experiments for the planar, swingup, and cartpole environments we start from a point randomly chosen from a region surrounding the start point in the underlying MDP; additionally, in the planar case we randomize the target every episode. In Table \ref{tab:main_results2}, we present results where the start and goal states are from fixed points, to see if there is improvement in the SOLAR results. Also we try to shorten the horizon for swingup to $100$ to see if shorter horizons can play a factor in domains with rather long horizons. We don't present any new results on the three-pole task as there was already a fixed starting state and fixed goal. We note that there is a dramatic improvement for the planar case when there is a fixed start and goal state and modest improvement in the cartpole and swingup cases. However, we still need to note that these results still are incomparable to the performance of any of CARL variants, offline CARL, online CARL, value-guided CARL, introduced in this paper.

\begin{table}[h]
\begin{small}
\centering
\begin{tabular}{|l | l | l | l | l|}
\hline
\textbf{Environment} & \textbf{Algorithm} & \textbf{Number of
Samples} & \textbf{Avg Result} & \textbf{Best Result} \\
\hline
Planar & SOLAR & 5000 (VAE) + 40000 (Control) & $26.70 \pm 5.92$ & $41\pm 7.28$\\
Cartpole & SOLAR & 10000 (VAE) + 40000 (Control) & $14.60\pm 1.70$ & $20.05 \pm 2.91$\\
Swingup & SOLAR & 20000 (VAE) + 40000 (Control) & $22.40 \pm 3.07$ & $34.03\pm 2.09$\\
\hline
\end{tabular}
\caption{Percentage of steps in goal state; averaged over all models and the best model. Additionally, the number of samples used for training for SOLAR are under the condition that there is the same start and same goal state for all episodes.}
\label{tab:main_results2}
\end{small}
\end{table}

\end{document}